\newif\ifconfver
\confvertrue        %declaring conference version true

\newif\ifonecoltab
\onecoltabtrue        % declaring one column version table size true
\newif\ifplainver  %declare a plain version
\plainvertrue
\ifplainver
\confverfalse                         %automatically disable conf. version argument if it's plain
\fi

\ifconfver
\documentclass[10pt,journal]{IEEEtran}
\else
\ifplainver
\documentclass[11pt]{article}
\usepackage{fullpage}
\else
\documentclass[12pt,draftcls,onecolumn]{IEEEtran}
\fi
\fi

\usepackage{calc,amsfonts,amssymb,amsmath,bm,url,color,theorem,graphicx,cite,epstopdf,nicefrac,stmaryrd}
\usepackage{psfrag,subfigure,float,epstopdf,bbold}
\usepackage[algoruled,linesnumbered]{algorithm2e}
\usepackage{multirow,comment}
\usepackage{algorithmic}
\usepackage[normalem]{ulem}
\usepackage{mdframed}
\usepackage{footnote}
\makesavenoteenv{tabular}
\definecolor{orange}{RGB}{255,107,0}

\usepackage{threeparttable}

\newtheorem{Fact}{Fact}

\theoremstyle{definition}

\newtheorem{Assumption}{Assumption}

\newcommand{\W}{\boldsymbol{W}}

\newcommand{\D}{\boldsymbol{D}}
\newcommand{\Q}{\boldsymbol{Q}}
\newcommand{\X}{\boldsymbol{X}}

\newcommand{\U}{\boldsymbol{U}}

\renewcommand{\H}{\boldsymbol{H}}

\newcommand{\A}{\boldsymbol{A}}

\newcommand{\Z}{\boldsymbol{Z}}
\newcommand{\x}{\boldsymbol{x}}

\renewcommand{\a}{\boldsymbol{a}}

\newcommand{\cone}[1]{\textsf{cone}\left\{#1\right\}}

\newcommand{\T}{{\!\top\!}}

\newcommand{\tX}{\underline{\bm X}}

\DeclareMathOperator*{\minimize}{\textrm{minimize}}

\newtheorem{theorem}{Theorem}
\newtheorem{lemma}{Lemma}
\newtheorem{proposition}{Proposition}

\newtheorem{definition}{Definition}
\newtheorem{remark}{Remark}
\newtheorem{proof}{Proof}
\begin{document}

	%--- I do things quite strangely here to accommodate three style modes.
	%--- input title and abstract here; it applies to all modes
	%--- it's too complex to do authors or they are input for each mode
	\newcommand{\papertitle}{
		%Robust Volume Minimization-Based Matrix Factorization for Remote Sensing and Document Clustering
		%A New Identifiability Result on Nonnegative Matrix Factorization
		Recovering Joint Probability of Discrete Random Variables from Pairwise Marginals
	}
	
	\newcommand{\paperabstract}{%
		Learning the joint probability of random variables (RVs) {is the cornerstone of} statistical signal processing and machine learning. However, direct nonparametric estimation for high-dimensional joint probability is in general impossible, due to the curse of dimensionality. Recent work has proposed to recover the joint probability mass function (PMF) of an arbitrary number of RVs from three-dimensional marginals, leveraging the algebraic properties of low-rank tensor decomposition and the (unknown) dependence among the RVs. Nonetheless, accurately estimating three-dimensional marginals can still be costly in terms of sample complexity, affecting the performance of this line of work in practice in the sample-starved regime. 
		Using three-dimensional marginals also involves challenging tensor decomposition problems whose tractability is unclear.
		This work puts forth a new framework for learning the joint PMF using only pairwise marginals, which naturally enjoys a lower sample complexity relative to the third-order ones. A coupled nonnegative matrix factorization (CNMF) framework is developed, and its joint PMF recovery guarantees under various conditions are analyzed.
		Our method also features a Gram--Schmidt (GS)-like algorithm that exhibits competitive runtime performance. The algorithm is shown to provably recover the joint PMF up to bounded error in finite iterations, under reasonable conditions.
		It is also shown that a recently proposed economical {\it expectation maximization} (EM) algorithm guarantees to improve upon the GS-like algorithm's output, thereby further lifting up the accuracy and efficiency.
		Real-data experiments are employed to showcase the effectiveness.
	}
	
	%--------
	
	\ifplainver
	
	\date{\today}
	
	\title{\papertitle}
	
	\author{
		 Shahana Ibrahim and Xiao Fu
		\\ ~ \\
		School of Electrical Engineering and Computer Science\\ Oregon State University\\
		Corvallis, OR 97331, United States
		\\~
	}
	
	\maketitle

	\begin{abstract}
		\paperabstract
	\end{abstract}
	
	\else
	\title{\papertitle}
	
	\ifconfver \else {\linespread{1.1} \rm \fi
		
		\author{Shahana Ibrahim and Xiao Fu
			
			\thanks{
				%	Manuscript received December 30, 2017; revised June 22, 2018 and October
				%	1, 2018; accepted October 2, 2018. Date of publication October 29, 2018; date
				%	of current version December 5, 2018. The associate editor coordinating the
				%	review of this manuscript and approving it for publication was Dr. Athanasios
				%	A. Rontogiannis.
				%	
				
				%This work is supported in part by National Science Foundation under Project NSF ECCS-1608961, ECCS 1808159, and III-1910118, the Army Research Office (ARO) under Proejct ARO W911NF-19-1-0247, and by the Chinese University of Hong Kong under the CUHK Direct Grant 4055113.
				
				X. Fu and S. Ibrahim are with the School of Electrical Engineering and Computer Science, Oregon State University, Corvallis, OR 97331, United States. email (xiao.fu,ibrahi)@oregonstate.edu

				This work is supported by the National Science Foundation under NSF-ECCS 1608961, NSF-ECCS 1808159, and NSF IIS-1910118, and the Army Research Office under ARO W911NF-19-1-0247 and W911NF-19-1-0407.

			}
		}
		
		\maketitle
		
		\ifconfver \else
		\begin{center} \vspace*{-2\baselineskip}
			%11th Revision, \today \\[2\baselineskip]
		\end{center}
		\fi
		
		\begin{abstract}
			\paperabstract
		\end{abstract}
		
\begin{IEEEkeywords}
	joint probability learning, nonnegative matrix factorization, probability tensors, two-dimensional marginals
\end{IEEEkeywords}
		
		\ifconfver \else \IEEEpeerreviewmaketitle} \fi
	
	\fi
	
	\ifconfver \else
	\ifplainver \else
	\newpage
	\fi \fi
	%---------------------------------------------------------------------------

	%---------------------------------------------------------------------------
\section{Introduction}
% Many learning and inference tasks in high-dimensional statistics boil down to estimating/approximating the joint probability of a set of random variables (RVs). 
{Estimating the joint probability of random variables (RVs) from noisy, incomplete and limited observations lies at the heart of statistical signal processing and machine learning. Having the joint probability allows us to construct statistically optimal estimators and detectors under certain metrics.
For example, the \textit{maximum a posteriori} (MAP) and the \textit{minimum mean squared error} (MMSE) principles can be easily carried out, if the joint probability of the RVs involved has been estimated \cite{trees2013detection}.
These principles are widely used in tasks such as speech processing, filter design, symbol detection, compressive sensing, and remote sensing; see, e.g., {\cite{ephraim1992bayesian,colavolpe2005map,rangan2012asymptotic,farag2005unified}}.
{In addition, joint probability is the key ingredient for estimating information-theoretic measures, e.g., mutual information and entropy \cite{cover1999elements}, which are often integrated into the solutions of challenging signal processing problems; see examples in blind source separation \cite{taleb1999source}, image processing \cite{thevenaz2000optimization} and radar waveform design \cite{yang2007mimo}.}
%\cite{taleb1999source}.
Joint probability estimation is also the linchpin of data mining, information retrieval, and machine intelligence \cite{murphy2012machine}.} 
{Driven by its critical role across multiple domains, joint probability learning using both parametric and nonparametric methods has been a long-term interest in the signal processing community for decades; see, e.g., \cite{yeredor2019maximum,kargas2017tensors,dabeer2013joint,hunsop2008newapproach,jenq1994nonparametric}.}

{Joint probability estimation poses a variety of challenges in both theory and methods. In particular,}
in the high-dimensional regime, directly estimating the joint probability via `structure-free' methods such as sample averaging is considered infeasible \cite{wainwright2019high}. Suppose that there are $N$ RVs, where each has an $I$-dimensional alphabet. To estimate the joint probability reliably, we generally need $S \gg \Omega(I^{N})$ `diverse' $N$-dimensional examples (without any entries missing). This is because the probability of encountering most $N$-tuples is very low. Therefore, only a small portion of the empirical distribution will be non-zero given a reasonable amount of data samples---this makes the approach considerably inaccurate in most cases. 

Many workarounds have been proposed to circumvent this challenge in the literature. For example, linear approximations have been widely used, e.g., the linear MMSE (LMMSE) estimator \cite{kay1993fundamentals}.
Logistic regression, kernel methods, and neural networks can be understood as nonlinear function approximation-based counterparts \cite{bishop2006pattern,murphy2012machine}. These are effective surrogates, but do not directly address the fundamental challenge in estimating high-dimensional joint probability from limited samples. Another commonly adopted route in statistical learning is to make use of problem-specific structural assumptions such as graphical models and prior distributions to help reduce the model complexity \cite{jordan1998learning,tipping2001sparse,blei2003latent}---but making these assumptions restricts the methods to specific applications.

Very recently, Kargas {\it et al.} proposed a new framework for {\it blindly} estimating the \textit{joint probability mass function} (PMF) of $N$ discrete finite-alphabet RVs \cite{kargas2017tensors}. Building upon a link between high-dimensional joint PMFs and low-rank tensors, the work in \cite{kargas2017tensors} shows that any $N$-dimensional joint PMF admits a naive Bayes {\it representation}. In addition, if the RVs are `reasonably dependent', the joint PMF can be recovered via jointly decomposing the three-dimensional marginal PMFs (which are third-order tensors). %{\blue Specifically, by assuming that the rank of the joint PMF tensor is below a certain threshold, recoverability from its three-dimensional marginals was shown. }

% {\blue The approach in \cite{kargas2017tensors} exploits the low-rank tensor structure of the joint PMF to establish recoverability}---if the three-dimensional marginal PMFs are accessible.

The work in \cite{kargas2017tensors} has shown promising results. However, many challenges remain. First, accurately estimating three-dimensional marginal PMFs is not a trivial task, which requires enumerating the co-occurrences of three RVs. This is particularly hard if the alphabets of the RVs are large and the data is sparse.
Second, tensor decomposition is a challenging optimization problem \cite{hillar2013most}. Factoring a large number of latent factor-coupled tensors as in \cite{kargas2017tensors} also gives rise to scalability issues. A natural question is: can we use pairwise marginals, i.e., joint PMFs of {\it two} RVs, to achieve the same goal of joint PMF recovery? Pairwise marginals are much easier to estimate relative to triple ones. In addition, the pairwise marginals give rise to probability mass matrices other than tensors---which may circumvent tensor decomposition in algorithm design and thus lead to more lightweight solutions.

Notably, a recent work in \cite{yeredor2019maximum} offers an {\it expectation maximization} (EM) algorithm that directly estimates the latent factors of the $N$th-order probabilistic tensor from the `raw data', instead of working with a large number of three-dimensional marginals. The EM algorithm is well-motivated, since it is associated with the {\it maximum likelihood estimator} (MLE) under the naive Bayes model in \cite{kargas2017tensors}. In addition, it admits simple and economical updates, and thus is much more scalable relative to the coupled tensor decomposition approach. However, {performance characterizations (e.g., estimation accuracy under finite samples and convergence properties) of the EM algorithm have been elusive}. 
% In addition, since the ML estimation problem is nonconvex, convergence properties of the EM algorithm are unclear.

\smallskip

\noindent
{\bf Contributions.}
In this work, we propose a new framework that utilizes pairwise marginals to recover the joint PMF of an arbitrary number of discrete finite-alphabet RVs.
We address both the theoretical aspects (e.g., recoverability) and the algorithmic aspects. Our contributions are as follows:

\noindent
$\bullet$ {\bf A Pairwise Marginal-based Framework.} On the theory side, we offer in-depth analyses for the recoverability of the joint PMF from pairwise marginals.
We show that the joint PMF can be recovered via jointly factoring the pairwise marginals under nonnegativity constraints in a latent factor-coupled manner.
Unlike \cite{kargas2017tensors} which leverages the uniqueness of tensor decomposition to establish recoverability, our analysis utilizes the model identifiability of \textit{nonnegative matrix factorization} (NMF)---leading to recoverability conditions that have a different flavor compared to those in \cite{kargas2017tensors}.
More importantly, accurately estimating pariwise marginals is much easier relative to the three-dimensional counterparts, in terms of the amount of data samples needed.

\noindent
$\bullet$ {\bf Provable Algorithms.} On the algorithm side, we propose to employ a block coordinate descent (BCD)-based algorithm for the formulated coupled NMF problem.
More importantly, we propose a pragmatic and easy-to-implement initialization approach, which is based on performing a Gram--Schmidt (GS)-like structured algorithm on a carefully constructed `virtual NMF' model. 
We show that this approach works provably well even when there is modeling noise---e.g., noise induced by finite-sample estimation of the pairwise marginals. This is in contrast to the method in \cite{kargas2017tensors}, whose recoverability guarantees are based on the assumption that the third-order marginals are perfectly accessible. 
We also show that the EM algorithm from \cite{yeredor2019maximum} can provably improve upon proper initialization (e.g., that given by the GS-like algorithm) towards the target latent model parameters that are needed to reconstruct the joint PMF.

\noindent
$\bullet$ {\bf Extensive Validation.} We test the proposed approach on a large variety of synthetic and real-world datasets. In particular, we validate our theory via performing classification on four different UCI datasets and three movie recommendation tasks.

\smallskip

Part of the paper was {presented at} the Asilomar signal processing conference 2020 \cite{ibrahim2020recover}. The journal version additionally includes new theoretical results, e.g., the recoverability analysis for the coupled NMF formulation, and a BCD-based CNMF algorithm, the detailed proofs for the Gram--Schmidt-like procedure, and the optimality analysis for the EM algorithm. More real data experiments are also presented.

\smallskip

\noindent
{\bf Notation.} We use $x,\bm x,\bm X,\underline{\X}$ to denote a scalar, vector, matrix and  tensor, respectively. $\mathbb{D}_{\rm KL}$ denotes the KL divergence. $\kappa(\bm X)$ denotes the condition number of the matrix $\bm X$ and is given by $\kappa(\bm X) = \frac{\sigma_{\max}(\bm X)}{\sigma_{\min}(\bm X)}$ where $\sigma_{\max}$ and $\sigma_{\min}$ are the largest and smallest singular values of $\bm X$, respectively. $\|\bm x\|_0$ denotes the `zero norm' of the vector $\bm x$, i.e., the number of nonzero elements in $\bm x$. $\bm X \ge \bm 0$ implies that all the elements of $\bm X$ are nonnegative. ${\rm cone}(\bm X)$ represents the conic hull formed by the columns of $\bm X$ and ${\rm conv}(\bm X)$ denotes the convex hull formed by the columns of $\bm X$. $\emptyset$ denotes the empty set. $\|\bm X\|_{2}$ represents the 2-norm of matrix $\bm X$ and $\|\bm X\|_{2} = \sigma_{\max}(\bm X)$.  $\|\bm X\|_{\rm F}$ denotes the Frobenius norm of $\bm X$. $\|\bm x\|_2$ and $\|\bm x\|_1$ denote $\ell_2$ and $\ell_1$ norm of vector $\bm x$, respectively. ${\rm vec}(\bm X)$ denotes the vectorization operator that concatenates the columns of $\bm X$. $^\T$ and $^\dag$ denote transpose and pseudo-inverse, respectively. $|{\cal C}|$ denotes the cardinality of set ${\cal C}$. ${\rm Diag}(\bm x)$ is a diagonal matrix that holds the entries of $\bm x$ as its diagonal elements. `$\odot$' denotes the Khatri-Rao product. {$[T]$ denotes $\{1,\dots,T\}$, where $T$ is an integer. $\bm e_n$ denotes the unit vector with $n$th element being one. $f(x) = O(g(x))$ (where $f(\cdot)$ and $g(\cdot)$ are nonnegative) means that there exist a positive constant $C$ and $x_0$ such that $f(x) \le Cg(x)$, $\forall x \ge x_0$. Similarly, $f(x) = \Omega(g(x))$ means that there exist a positive constant $c$ and $x_0$ such that $f(x) \ge cg(x)$, $\forall x \ge x_0$.}

\section{Background}
Consider a set of discrete and finite-alphabet RVs, i.e., $Z_1,\ldots,Z_N$. We will use ${\sf Pr}(i_1,\ldots,i_N)$ as the shorthand notation to represent $${\sf Pr}\left(Z_1=z_1^{(i_1)},\ldots,Z_N=z_N^{(i_N)}\right)$$ in the sequel, where $\{z_n^{(1)},\ldots,z_n^{(I_n)}\}$ denotes the alphabet of $Z_n$. 
Suppose that we do not have access to the joint PMF ${\sf Pr}(i_1,\ldots,i_N)$. 
Instead, we have access to the lower-dimensional marginal distributions, e.g., ${\sf Pr}(i_j,i_k,i_\ell)$ and ${\sf Pr}(i_j,i_k)$ for different $j,k,\ell$. Kargas \textit{et al.} asked the following research question \cite{kargas2017tensors}:

\vspace{.15cm}
\noindent
\boxed{
	\begin{minipage}{.98\linewidth}
		\noindent
		Can the joint distribution of $Z_1,\ldots,Z_N$ be recovered from the low-dimensional marginals without knowing any structural information about the underlying probabilistic model?
	\end{minipage}
}
\vspace{.15cm}

Note that estimating a high-dimensional joint distribution from lower-dimensional marginals is challenging, but not entirely impossible---if some structural information between the RVs (e.g., tree or Markov chain) is known; see discussions in \cite{kargas2017tensors}. 
However, learning the structural information from data {\it per se} is often a hard problem \cite{jordan1998learning}.

\subsection{Tensor Algebra} 
The work in \cite{kargas2017tensors} utilizes a connection between joint PMFs and low-rank tensors under the \textit{canonical polyadic decomposition} (CPD) model \cite{harshman1970foundations,sidiropoulos2017tensor}.  
If an $N$-way tensor $\underline{{\bm X}} \in \mathbb{R}^{I_1 \times I_2 \times \cdots \times I_N}$ has CP rank $F$, 
it can be written as:
\begin{equation}\label{eq:individ}
\underline{{\bm X}}(i_1,i_2,\ldots,i_N) = \sum_{f=1}^F \boldsymbol{\lambda}(f)  \prod_{n=1}^N {\bm A}_n(i_n,f),
\end{equation}
where  ${\bm A}_n \in \mathbb{R}^{I_n \times F}$ is called the mode-$n$ latent factor.
In the above, ${\bm \lambda}=[ \bm \lambda(1),\ldots,\bm \lambda(F)]^\T$ with $\|{\bm \lambda}\|_0=F$ is employed to `absorb' the norms of columns (so that the columns of $\A_n$ can be normalized w.r.t. a certain $\ell_q$ norm for $q\geq 1$). 
We use  the shorthand notation $$\underline{{\bm X}} = [\![ {\bm \lambda}, {\bm A}_1,\ldots,{\bm A}_N ]\!]$$ to denote the CPD in \eqref{eq:individ}.
Note that $F\gg I_n$ for $n=1,\ldots,N$ could happen. In fact, the tensor rank can reach $O(I^{N-1})$ if $I_1=\ldots=I_N=I$ \cite{sidiropoulos2017tensor}.

\subsection{Joint PMF Recovery: A Tensor Perspective} 
A key takeaway in \cite{kargas2017tensors} is that {\em any} joint PMF admits a naive Bayes (NB) model {\em representation}; i.e., any joint PMF can be generated from a latent variable model with just one hidden variable. It follows that the joint PMF of $\{Z_n\}_{n=1}^N$ can always be decomposed as
\begin{equation}\label{eq:pmf_latent_var}
\begin{aligned}
{\sf Pr}(i_1,i_2,\ldots,i_N) = \sum_{f=1}^F {\sf Pr}(f) \prod_{n=1}^N {\sf Pr}(i_n | f)
\end{aligned},
\end{equation}
where ${\sf Pr}(f) := {\sf Pr}(H = f) $ is the prior distribution of a latent variable $H$ and ${\sf Pr}(i_n | f) := {\sf Pr}( Z_n = z_n^{(i_n)} | H = f) $ are the conditional distributions. 
Remarkably, such representation is {\em universal} if one allows the hidden variable to have a sufficiently rich finite alphabet:
\begin{theorem}\label{prop:F_upperbound}\cite{kargas2017tensors}
	The maximum $F$ needed in order to describe an arbitrary PMF using a naive Bayes model is bounded by 
	$F\leq \underset{k}{\min}\prod_{\substack{n=1 \\ n \neq k}}^N I_n.$
\end{theorem} 
To see this, one can represent any joint PMF as an $N$th-order tensor by letting $\tX(i_1,\ldots,i_N)={\sf Pr}(i_1,\ldots,i_N)$ and  ${\bm A}_n(i_n,f) = {\sf Pr}(i_n | f),~\boldsymbol{\lambda}(f) = {\sf Pr}(f).$
Since any nonnegative tensor admits a nonnegative polyadic decomposition with nonnegative latent factors, any joint PMF admits a naive Bayes representation with finite latent alphabet.

%\begin{figure}[t!]
%	\centering
%	\includegraphics[width=1\linewidth]{figures/marginals.pdf}
%	\caption{Illustration of Kargas \textit{et al.}'s framework for joint PMF recovery using three-dimensional marginals \cite{kargas2017tensors}.}
%	\label{fig:marginals}
%\end{figure}

A relevant question is: what does `low tensor rank' mean for the joint PMF tensors? In \cite{kargas2017tensors}, the authors argue that reasonably dependent RVs lead to a joint PMF tensor with relatively small $F$---which is normally the case in statistical learning and inference. For example, classification is concerned with inferring $Z_N$ (label) from $Z_1,\ldots,Z_{N-1}$ (features), and the dependence amongs the RVs are leveraged for building up effective classifiers.

Another key observation in \cite{kargas2017tensors} is that the marginal distribution of any subset of RVs is an induced CPD model. From the law of total probability,
if one marginalize the joint PMF down to ${\sf Pr}(i_j,i_k,i_\ell)$ $\forall j,k,\ell \in [N],$ the expression is
$${\sf Pr}(i_j,i_k,i_\ell) = \sum_{f=1}^{F}   {\sf Pr}(f) {\sf Pr}(i_j | f) {\sf Pr}(i_k | f) {\sf Pr}(i_\ell | f).$$
Let $\underline{\bm X}_{jk\ell}(i_j,i_k,i_\ell)={\sf Pr}(i_j,i_k,i_\ell)$. Then, we have $\underline{{\bm X}}_{jk\ell} = [\![ \boldsymbol{\lambda},{\bm A}_j,{\bm A}_k,{\bm A}_\ell ]\!]$, 
where $\{{\bm A}_n\}_{n=1}^N$ and ${\bm \lambda}$ are defined as before. Another important observation here is that
{\it the marginal PMFs $\underline{\bm X}_{jk\ell}$'s and the joint PMF $\underline{\bm X}$ share the same $\bm A_{j}$, $\A_{k}$, $\bm A_{\ell}$ and $\bm \lambda$}.
It is readily seen that if the $\underline{\bm X}_{jk\ell}$'s admit essentially unique CPD, then $\bm A_n$'s and $\bm \lambda$ can be identified from the marginals---and thus ${\sf Pr}(i_1,\ldots,i_N)$ can be recovered via identifying the latent factors. 

The tensor-based approach in \cite{kargas2017tensors} has shown that recovering high-dimensional joint PMF from low-dimensional marginals is possible. However, a couple of major hurdles exist for applying it in practice. First, estimating three-dimensional marginals ${\sf Pr}(i_j,i_k,i_\ell)$ is still not easy, since one needs many co-realizations of three RVs. For sparse datasets, this is particularly hard. Second, tensor decomposition is a hard computation problem \cite{hillar2013most}. The coupled tensor decomposition (CTD) algorithm in \cite{kargas2017tensors} that involves many tensors is even more challenging in practice.

\subsection{Maximum Likelihood Estimation and EM} \label{sec:mle_em}

A recent work in \cite{yeredor2019maximum} offers an alternative for estimating ${\sf Pr}(f)$ and ${\sf Pr}(i_n|f)$.
Instead of working with third-order statistics, the model estimation problem is formulated as a maximum likelihood (ML) estimation problem in \cite{yeredor2019maximum}, and an EM algorithm is employed to handle the ML estimator (MLE).
The upshot of this approach is that the EM updates are much more scalable relative to the CTD updates in \cite{kargas2017tensors}. 
However, there are two challenges of applying EM. First, {the recoverability guarantees of the joint PMF using the MLE framework is unclear under finite number of observations}.
Second, as the authors of \cite{yeredor2019maximum} have noticed, the EM algorithm sometimes does not converge to a good solution if randomly initialized. 
This is perhaps due to the nonconvex nature of the ML estimation problem. %The recent work in \cite{yeredor2019maximum} has also proposed an \textit{alternating directions} (AD)-based algorithm for handling the MLE. 
%\reminder{we did not talk about MLE based alternating direction algorithm in Yeredors. Should we add it here?}

\section{Proposed Approach}
To advance the task of joint PMF recovery from marginal distributions, we propose a pairwise marginal-based approach.
Note that the pairwise marginal is given by
\begin{align*}
& {\sf Pr}(i_j,i_k)=\sum_{f=1}^F{\sf Pr}(f){\sf Pr}(i_j|f){\sf Pr}(i_k|f) \\&  \Longleftrightarrow
\X_{jk} = \A_j{\bm D}(\bm \lambda)\A_k^\T,~\X_{jk}(i_j,i_k)={\sf Pr}(i_j,i_k),
\end{align*}    
where $\bm D(\bm \lambda) = {\rm Diag}(\bm \lambda)$. 

In practice, the pairwise marginals ${\bm X}_{jk}$'s are estimated from realizations of the joint PMF. Consider a set of realizations (data samples) of ${\sf Pr}(Z_1,\ldots,Z_N)$, denoted as $\{\bm d_{s} \in\mathbb{R}^{N} \}_{s=1}^S$. 
% {\blue The following sample averaging estimator can be employed:
% \begin{align*}
% \widehat{\bm X}_{jk}(i_j,i_k) &= \frac{1}{|\mathcal{S}_{j,k}|}\sum_{s\in\mathcal{S}_{j,k}} \mathbb{I}\left[\bm d_s(j)=z_j^{(i_j)},\bm d_s(k)=z_k^{(i_k)}\right],
% \end{align*}
% where $\mathcal{S}_{j,k} \subseteq [S]$ denotes the set of indices of the data samples in which the realizations of both $Z_j$ and $Z_k$ are observed,} 
%Assuming that there are no missing observations {\blue (i.e., every RV is observed with a probability $p=1$)}, the following sample averaging estimator can be employed:
%\begin{align*}
%\widehat{\bm X}_{jk}(i_j,i_k) &= \frac{1}{S}\sum_{s=1}^S \mathbb{I}\left[\bm d_s(j)=z_j^{(i_j)},\bm d_s(k)=z_k^{(i_k)}\right],
%\end{align*}
%where $\bm d_s(n)$ denotes the realization of $Z_n$ in the $s$-th sample and $\mathbb{I}[E]=1$ if the event $E$ happens and $\mathbb{I}[E]=0$ otherwise. 
{The following sample averaging estimator can be employed:
 \begin{align*}\label{eq:Xij_est}
 \widehat{\bm X}_{jk}(i_j,i_k) &= \frac{1}{|\mathcal{S}_{jk}|}\sum_{s\in\mathcal{S}_{jk}} \mathbb{I}\left[\bm d_s(j)=z_j^{(i_j)},\bm d_s(k)=z_k^{(i_k)}\right],
 \end{align*}
 where $\bm d_s(n)$ denotes the realization of $Z_n$ in the $s$-th sample, $\mathbb{I}[E]=1$ if the event $E$ happens and $\mathbb{I}[E]=0$ otherwise,
 $S$ is the number of co-realizations of $Z_1,\ldots,Z_N$ (i.e., the number of `samples'), and
 $\mathcal{S}_{jk} \subseteq [S]$ denotes the set of indices of the data samples in which the realizations of both $Z_j$ and $Z_k$ are observed.}

{ 
%Using second-order marginals as the building blocks for joint PMF recovery is well-motivated.
Notably, with the same amount of data (i.e., with a fixed $S$), the second-order statistics can be estimated to a much higher accuracy, compared to the third-order ones \cite{han2015minimax}. 
This is particularly articulated when the dataset is sparse.
To see this, assume that every RV has an alphabet size $I$ and is observed with probability {$p \in (0,1]$} (which means that every entry of $\bm d_s$ is observed with probability $p$---and $\bm d_s$ is very sparse when $p$ is small). Then, on average $Sp^2$ and $Sp^3$ co-realizations of every pair and triple of RVs, respectively, can be observed in the samples.
Under such cases, Theorem~1 in \cite{han2015minimax} shows that the average estimation error (over all possible joint distributions) for $\underline{\bm X}_{jk\ell}$ in terms of $\ell_1$ distance is $O(\sqrt{\nicefrac{I^3}{Sp^3}})$,
while that for $\X_{jk}$ is $O( \sqrt{\nicefrac{I^2}{Sp^2}})$. One can see that the difference is particularly significant when $p$ is small---i.e., when the dataset is considerably sparse.}

\subsection{Challenges}
Despite the sample complexity appeal of using pairwise marginals, the {\it recoverability} of the joint PMF under such settings is nontrivial to establish. The main analytical tool for proving recoverability in \cite{kargas2017tensors} is the essential uniqueness of tensor decomposition---i.e., the latent factors of low-rank tensors are identifiable up to resolvable ambiguities under mild conditions. 
However, pairwise distributions such as $\bm X_{jk}= \A_j{\bm D}(\bm \lambda)\A_k^\T$ are matrices, and low-rank matrix decomposition is in general \textit{nonunique}. 
A natural thought to handle the identifiability problem would be employing certain NMF tools \cite{fu2018nonnegative,gillis2014and}, since the latent factors are all nonnegative, per their physical interpretations.
However, the identifiability of NMF models holds only if $F\leq \min\{I_j,I_k\}$ (and preferably $F\ll\min\{I_j,I_k\}$).
The pairs $\bm X_{jk}= \A_j{\bm D}(\bm \lambda)\A_k^\T\in\mathbb{R}^{I_j\times I_k}$ inherit the inner dimension $F$ (i.e., the column dimension of $\bm A_j$) from the joint PMF of all the variables, which is the tensor rank of an $N$th-order tensor. As mentioned, the tensor rank $F$ could be much larger than the $I_j$'s. 
Hence, one may not directly use the available NMF uniqueness results on individual $\X_{jk}$'s to argue for joint PMF recoverability.
Nonetheless, as we will show shortly, NMF identifiability can be applied onto carefully constructed {\it coupled} NMF models to establish recoverability.

\subsection{Preliminaries: NMF and Model Identifiablility}
To see how we approach these challenges, let us first briefly introduce some pertinent preliminaries of NMF.
Consider a nonnegative data matrix $\X\in\mathbb{R}^{L\times K}$. Assume that the matrix is generated by $\X=\W\H^\T$, where $\W\in\mathbb{R}^{L\times F}$ and $\H \in \mathbb{R}^{K\times F}$. Also assume that $\W\geq \bm 0$ and $\H\geq \bm 0$ and $F\leq \min\{L,K\}$. NMF aims to factor $\X$ for identifying $\W$ and $\H$ up to certain trivial ambiguities. 

A number of conditions on $\W$ and/or $\H$ allow us to establish identifiability of the latent factors. The first one is the so-called separability condition:
\begin{definition} (Separability) \label{def:sep} \cite{donoho2003does}
	If $\bm H\geq \bm 0$, and  $\bm \varLambda =\{ l_1,\ldots,l_F \}$ such that $\bm H(\bm \varLambda,:)=\bm \Sigma$ holds, where $\bm \Sigma ={\rm Diag}(\alpha_1,\ldots,\alpha_F)$ and $\alpha_f>0$, then, $\H$ satisfies the {\it separability condition}. When $\bm \varLambda =\{ l_1,\ldots,l_F \}$ satisfies $\|\H(l_f,:)-\bm e_f\|_2\leq \varepsilon$ for $f=1,\ldots,F$, $\bm H$ is called $\varepsilon$-separable.
\end{definition}
If one of $\W$ and $\H$ satisfies the separability condition and the other has full column rank, there exist polynomial-time algorithms that can provably identify $\W$ and $\H$ up to scaling and permutation ambiguities. 
Among these algorithms, a number of Gram--Schmidt-like lightweight algorithms exist; see, e.g., \cite{Gillis2012,fu2014self,arora2012practical,VMAX,MC01}.
A more relaxed condition is as follows:
\begin{definition}\label{def:suff}  (Sufficiently Scattered) \cite{huang2014non,fu2015blind,fu2016robust}
	Assume that $\bm H\geq \bm 0$ and 
	${\cal C}\subseteq {\rm cone}\{\bm H^\T\}$
	where 	$\mathcal{C} = \{\x \in\mathbb{R}^F~|~ \x^\T\mathbf{1} \geq \sqrt{F-1}\|\x\|_2\}$. 
	In addition, assume that ${\sf cone}\{\H^\T\} \not\subseteq {\sf cone}\{\bm Q\} $ for any orthonormal {$\Q\in\mathbb{R}^{F\times F}$} except for the permutation matrices.
	Then, $\bm H$ is called {\it sufficiently scattered}.
\end{definition}
The sufficiently scattered condition subsumes the separability condition as a special case. An important result is that if $\W$ and $\H$ are both sufficiently scattered, then the model $\X=\W\H^\T$ is unique up to scaling and permutation ambiguities \cite{huang2014non}. The two conditions are illustrated in Fig.~\ref{fig:cone}. 
More detailed discussions and illustrations can be seen in \cite{fu2018nonnegative,fu2018identifiability}.

In principle, the `taller' $\bm H$ is, the chance that $\bm H$ (approximately) attains the separability or sufficiently scattered condition is higher, if $\H$ is generated following a certain continuous distribution. This intuition was formalized in \cite{ibrahim2019crowdsourcing}, under the assumption that $\H(l,:)$'s are drawn from the probability simplex uniformly at random and then positively scaled. The work in \cite{ibrahim2019crowdsourcing} also shows that the sufficiently scattered condition is easier to satisfy relative to separability, under the same generative model; see Lemma~\ref{lem:Lm} in Appendix~\ref{app:lemmata}.

%\begin{comment}

\begin{figure}[t!]
	\centering
	\includegraphics[scale= 0.15]{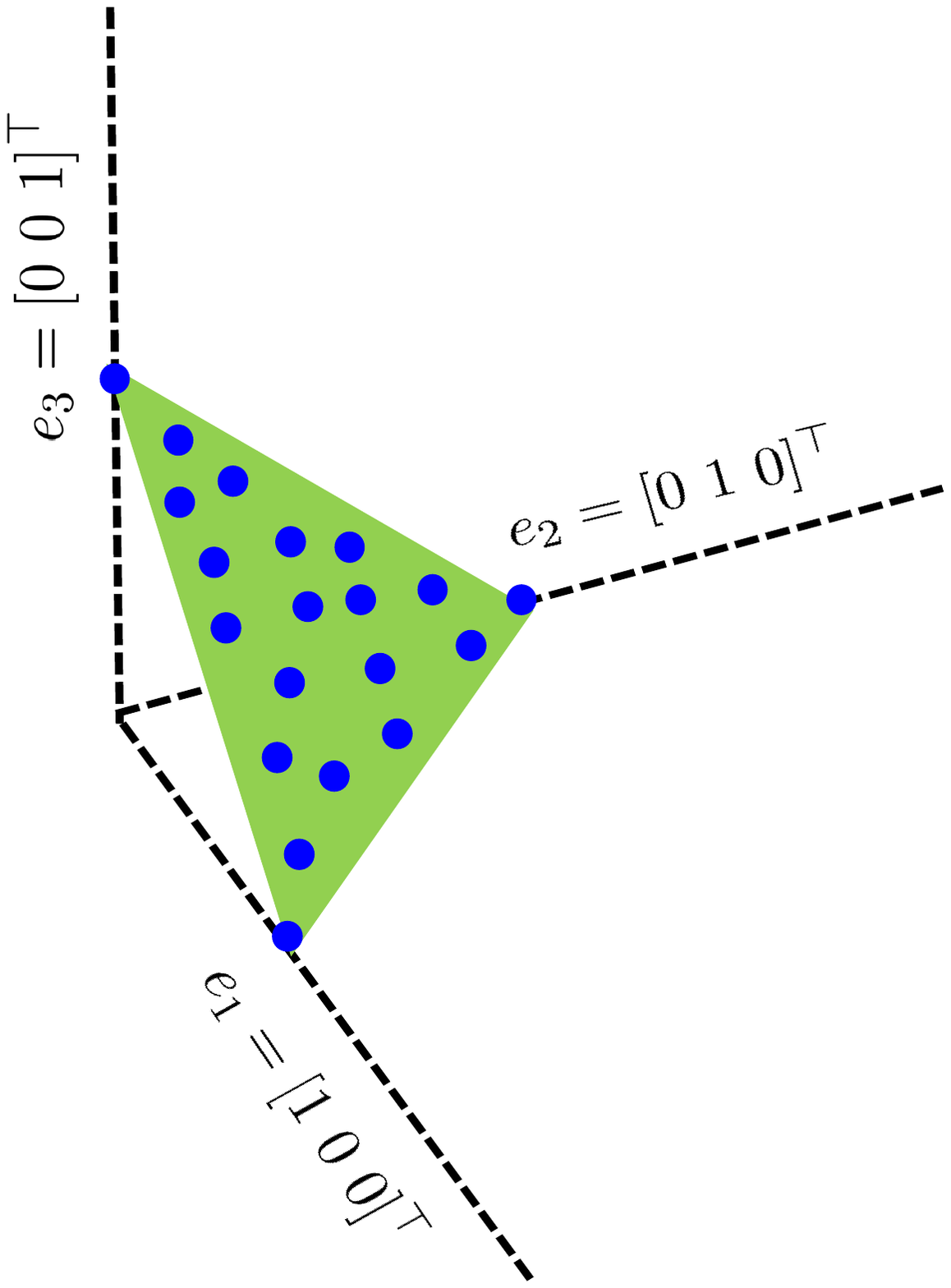}
	\includegraphics[scale= 0.15]{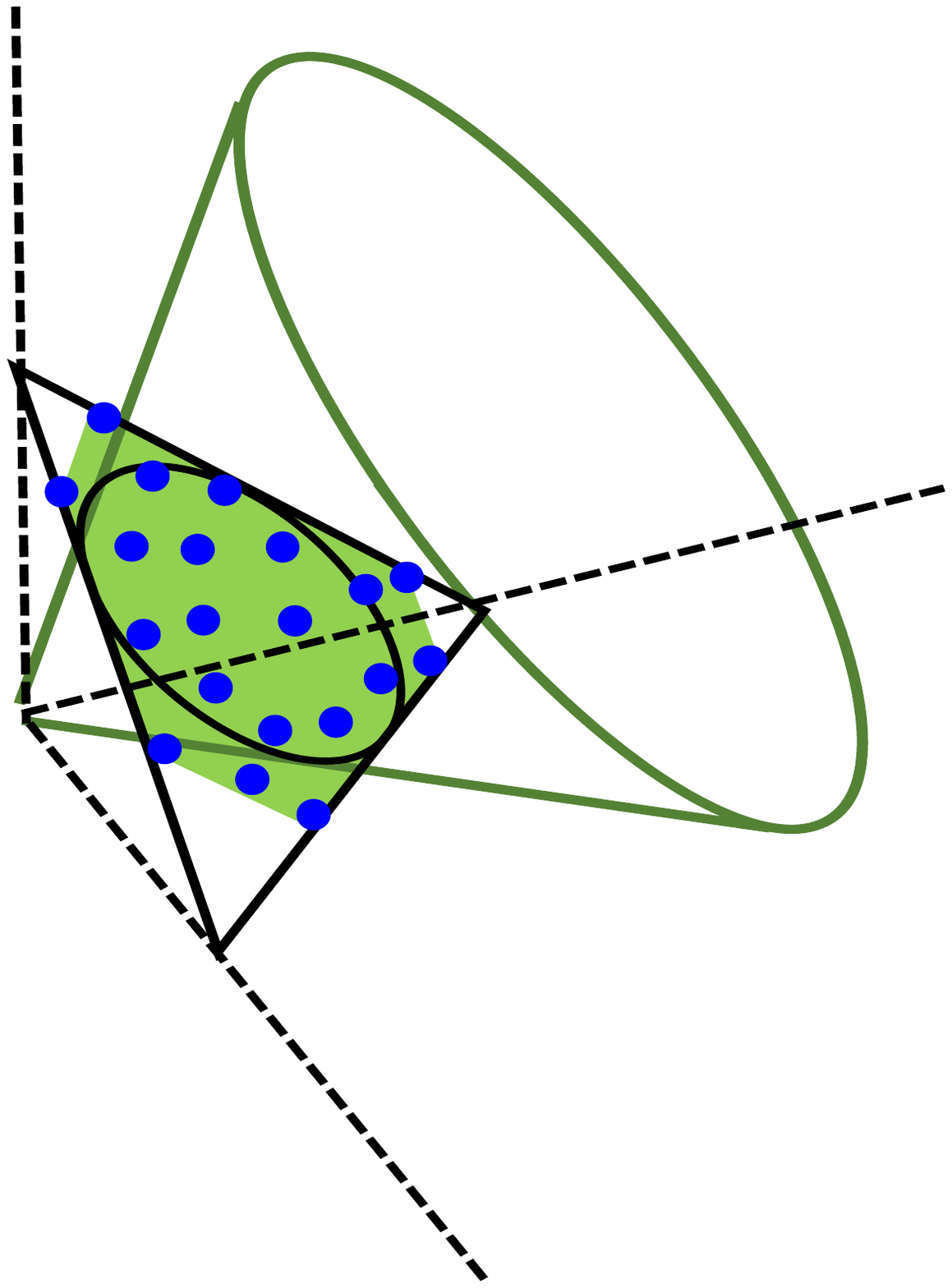}
	\put(-24,100){$\mathcal{C}$}
	\caption{{Left: an $\bm H \in \mathbb{R}^{20 \times 3}$ that satisfies the separability condition. Right: an $\H \in \mathbb{R}^{20 \times 3}$ that satisfies the sufficiently scattered condition.} The triangle corresponds to probability simplex $\Delta_3$; and the dots correspond to rows of $\H$ rescaled to have unit $\ell_1$-norm. The inner circle on $\Delta_3$ corresponds to the intersection of the second-order cone ${\cal C}$ and $\Delta_3$; and the shaded region is the intersection of $\cone{\bm H^\T}$ and $\Delta_3$.}
	\label{fig:cone}
	\vspace{-.35cm}
\end{figure}
%\end{comment}

\subsection{Recoverability Analysis}
Our goal is to identify $\A_n$ and $\bm \lambda$ from the available pairwise marginals $\X_{jk}=\A_j\bm D(\bm \lambda)\A_k^\T$'s, so that the joint PMF can be reconstructed following \eqref{eq:pmf_latent_var}.
Note that applying existing NMF techniques onto $\X_{jk}$ is unlikely to work, since when $F> \min \{I_j,I_k\}$, `fat' $\A_j\in\mathbb{R}^{I_j\times F}$'s can never be separable or sufficiently scattered. To circumvent this, our idea is as follows.
Consider a splitting of the indices of the $N$ variables, i.e., ${\cal S}_1=\{\ell_1,\ldots,\ell_M\}$ and ${\cal S}_2=\{\ell_{M+1},\ldots,\ell_{N}\}$ such that
${\cal S}_1\cup{\cal S}_2=\{1,\ldots,N\}$ and ${\cal S}_1\cap {\cal S}_2=\emptyset.$
Then, we construct the following matrix:
\begin{equation}\label{eq:constructedX}
\begin{aligned}
\widetilde{\bm X}=& \begin{bmatrix} \bm X_{\ell_1\ell_{M+1}}& \ldots & {\bm X}_{\ell_1\ell_N} \\ \vdots & \vdots &\vdots \\ {\bm X}_{\ell_M\ell_{M+1}} & \ldots & {\bm X}_{\ell_M\ell_N} \end{bmatrix} \\
= &\underbrace{\begin{bmatrix}
	\bm A_{\ell_1}\\ \vdots\\ \bm A_{\ell_M}
	\end{bmatrix}}_{\bm W}\underbrace{{\bm D}({\bm \lambda })[{\bm A}_{\ell_{M+1}}^\T,\ldots,{\bm A}_{\ell_{N}}^\T]}_{{\bm H}^\T}.
\end{aligned}
\end{equation} 
Note that in this `virtual NMF' model, $\W$ and $\H$ matrices have sizes of $MI\times F$ and $(N-M)I\times F$, respectively, if $I_1=\ldots=I_N=I$. 
The idea is to construct $\widetilde{\bm X}$ such that $F\leq\min\{MI, (N-M)I \}$ so that $\W$ and $\H$ may satisfy the separability or sufficiently scattered condition---and thus $\A_1,\ldots,\A_N$
can be identified via factoring $\widetilde{\X}$.
This is a straightforward application of the NMF identifiability results introduced in the previous subsection.

The challenge lies in finding a suitable splitting of ${\cal S}_1,{\cal S}_2$ such that $\W$ and $\H$ are sufficiently scattered and is highly nontrivial. 
Note that even if $\W$ and $\H$ are given, checking if these matrices are sufficiently scattered is hard \cite{huang2014non}.
It is also impractical to try every possible index splitting. 
To address this challenge, we consider the following factorization problem:
\begin{subequations}\label{eq:coupled_mat}
	\begin{align}
	{\rm find}&~\A_n,~n=1,\ldots,N,~\bm \lambda\\
	{\rm subject~to}&~\X_{jk}=\A_j\bm D(\bm \lambda)\A_k^\T,~\forall  j,k \in{\bm \Omega} \label{eq:fic_feas}\\
	&~\bm 1^\T\bm A_j=\bm 1^\T,~\bm A_j\geq \bm 0\\
	&~\bm 1^\T\bm \lambda =1,~\bm \lambda\geq{\bm 0},
	\end{align}
\end{subequations}
where $\bm \Omega$ contains the index set of $(j,k)$'s such that $j<k$ and the joint PMF ${\sf Pr}(i_j,i_k)$ is accessible. 
Note that the above can be understood as a latent factor-coupled NMF problem.
Regarding the identifiability of the conditional PMFs and the latent prior, we have the following result:
\begin{theorem}[Recoverability]\label{prop:prop2}
	Assume that that ${\sf Pr}(i_j,i_k)$'s for $j,k\in{\bm \Omega}$ are available and that ${\sf Pr}(f)\neq 0$ for $f=1,\ldots,F$.
	Suppose that there exists ${\cal S}_1=\{\ell_1,\ldots,\ell_M\}$ and ${\cal S}_2=\{\ell_{M+1},\ldots,\ell_Q\}$ such that $Q\leq N$
	and $ {\cal S}_1\cup{\cal S}_2\subseteq \{1,\ldots,N\},~{\cal S}_1\cap {\cal S}_2=\emptyset. $
	Also assume the following conditions hold:\\
	\textbf{(i)} the matrices $[\A_{\ell_1}^\T,\ldots,\A_{\ell_M}^\T]^\T$ and $[\A_{\ell_{M+1}}^\T,\ldots,\A_{\ell_{Q}}^\T]^\T$ are \textit{sufficiently scattered};\\
	\textbf{(ii) }all pairwise marginal distributions ${\sf Pr}(i_j,i_k)$'s for $j\in{\cal S}_1$ and $k\in{\cal S}_2$ are available; \\
	\textbf{(iii)} every $T$-concatenation of $\A_n$'s, i.e., $[\A_{n_1}^\T,\ldots,\A_{n_T}^\T]^\T$, is a full column rank matrix, if $I_{n_1}+\ldots+I_{n_T}\geq F$;\\
	\textbf{(iv) }for every $j\notin{\cal S}_1\cup{\cal S}_2$ there exists a set of $r_t\in{\cal S}_1\cup{\cal S}_2$ for $t=1,\ldots,T$ such that ${\sf Pr}(i_j,i_{r_t})$ or ${\sf Pr}(i_{r_t},i_j)$ are available.
	
	Then, solving Problem~\eqref{eq:coupled_mat} recovers ${\sf Pr}(i_j|f)$ and ${\sf Pr}(f)$ for $j=1,\ldots,N, ~f=1,\dots, F$, thereby the joint PMF ${\sf Pr}(i_1,\ldots,i_N)$.
\end{theorem}

The proof is relegated to Appendix~\ref{app:prop1}, which is reminiscent of identifiability of a similar coupled NMF problem that arises in crowdsourced data labeling \cite{ibrahim2019crowdsourcing}, with proper modifications to handle the cases where $F>I_n$ (since the crowdsourcing model always has $F=I_n$, which is a simpler case). 
In particular, conditions (iii) and (iv) are employed to accommodate the more challenging cases.
The criterion spares one the effort for first finding ${\cal S}_1$ and ${\cal S}_2$ and then constructing the matrix $\widetilde{\X}$. Instead, under Theorem~\ref{prop:prop2}, the goal is simply finding an approximate solution of \eqref{eq:coupled_mat}---which is much more approachable in practice.

\begin{remark}
{ 
Under our pairwise marginal based approach, a key enabler for recovering the joint PMF is the {\it identifiability} of $\A_n$'s up to an {\it identical} column permutation ambiguity.
At first glance, the identifiability of the latent factors seems to be irrelevant since what we truly care {about} is the ambient joint PMF tensor $\tX\in\mathbb{R}^{I_1\times\ldots\times I_N}$.
However, since the tensor is heavily marginalized (i.e.,     one only observes pairwise marginals), the identifiability of $\bm A_n$'s is leveraged on to assemble the joint PMF following the model in \eqref{eq:pmf_latent_var} under our framework. 
}
\end{remark}

\section{Algorithms and Performance Analyses}
In this section, we develop algorithms under the coupled NMF framework and analyze performance under realistic conditions, e.g., cases with finite samples and missing values.
\subsection{BCD for Problem~\eqref{eq:coupled_mat}} \label{sec:alg1}
To handle Problem \eqref{eq:coupled_mat}, we employ the standard procedure based on \textit{block coordinate descent} (BCD) for handling coupled tensor/matrix decomposition algorithms in the literature \cite{kargas2017tensors,kargas2019learning,ibrahim2019crowdsourcing}.
To be specific, we cyclically minimize the constrained optimization problem w.r.t. $\A_k$, when fixing $\A_j$ for all $j\neq k$ and $\bm \lambda$. Consequently, the subproblem is as follows:
\begin{subequations}\label{eq:sub}
	\begin{align}
	\minimize_{\A_k}&~\sum_{j\in{\bm \Omega}_k}\mathbb{D}\left( \X_{jk},\A_j\bm D(\bm \lambda)\A_k^\T\right)\\
	{\rm subject~to}&~\bm 1^\T\bm A_k=\bm 1^\T,~\bm A_k\geq \bm 0,
	\end{align}
\end{subequations}
where ${\bm \Omega}_k$ is the index set of $j$ such that ${\sf Pr}(i_j,i_k)$ is available and $\mathbb{D}(\cdot,\cdot)$ is a certain `distance' measure. Under a number of { $\mathbb{D}(\cdot,\cdot)$-functions}, e.g., the Euclidean distance or the Kullback--Leibler (KL) {divergence}, the subproblems are convex and thus can be solved via off-the-shelf algorithms.
In this work, we employ the KL divergence as the {fitting} measure since it is natural for measuring {the similarity of} PMFs.
Consequently, the subproblems are
\begin{subequations}\label{eq:subkl}
	\begin{align}
	\A_k &\leftarrow\arg\min_{ \begin{subarray}{c} \bm 1^\T\bm A_k=\bm 1^\T\\\bm A_k\geq \bm 0\end{subarray}}~\sum_{j\in{\bm \Omega}_k}\mathbb{D}_{\sf KL}\left( \X_{jk},\A_j\bm D(\bm \lambda)\A_k^\T\right), \label{eq:Aupdate}\\
	\bm \lambda &\leftarrow\arg\min_{ \begin{subarray}{c} \bm 1^\T\bm \lambda=1\\ \bm \lambda\geq \bm 0\end{subarray}}~\sum_{j,k\in{\bm \Omega}}\mathbb{D}_{\sf KL}\left( \X_{jk},\A_j\bm D(\bm \lambda)\A_k^\T\right). \label{eq:lambda_update}
	\end{align}
\end{subequations}
%Note that both of the above subproblems are convex, and can be solved by off-the-shelf solvers. 
We employ the {\it mirror descent} algorithm to solve the two subproblems. Mirror descent admits closed-form updates for handling KL-divergence based linear model fitting under simplex constraints (which is also known as exponential gradient in the literature); see, e.g., \cite{arora2012practical}. The algorithm {CNMF-OPT} is described in \eqref{algo:CNMF-OPT}.
\begin{algorithm}[h!]
	\footnotesize
	\SetKwInOut{Input}{input}
	\SetKwInOut{Output}{output}
	\SetKwRepeat{Repeat}{repeat}{until}
	%\SetKwRepeat{Repeat}{for}{end}
	\Input{data samples $\{\bm d_s\}_{s=1}^S$}
	
	 estimate second order statistics $\widehat{\X}_{jk }$;
	 
	 get initial estimates $\{\widehat{\A}_n\}_{n=1}^N$ and $\widehat{\bm \lambda}$ (e.g., using {CNMF-SPA});
	 %{$t=1$ {\bf to} ${MaxIter}$}
		\Repeat{some stopping criterion is reached}{
		
		\For{$n=1$ {\bf to} $N$}{
		
		  update $\widehat{\A}_n \leftarrow \eqref{eq:Aupdate}$ using mirror descent;
		}
		
		 update $\widehat{\bm \lambda} \leftarrow\eqref{eq:lambda_update}$ using mirror descent;
		}
		\Output{estimates $\{\widehat{\A}_n\}_{n=1}^N$, $\widehat{\bm \lambda}$}
			\caption{{CNMF-OPT}}\label{algo:CNMF-OPT}
\end{algorithm}

Notably, the work in \cite{kargas2017tensors} also has some experiments (without recoverability analysis) using pairwise marginals to recover the joint PMF using a similar formulation (using Euclidean distance).
Our rationale for reaching the formulation in \eqref{eq:coupled_mat} is very different, which is motivated by the NMF theory.
Also, the result using pairwise marginals in \cite{kargas2017tensors} seems not promising. 
In this work, our analysis in Theorems~\ref{prop:prop2} has suggested that the reason why pairwise marginals do not work in \cite{kargas2017tensors} is likely due to optimization pitfalls, rather than lack of recoverability.
In the next subsection, we show that with a carefully designed initialization scheme, accurately recovering joint PMFs from pairwise marginals is viable.

\subsection{Key Step: Gram--Schmidt-like Initialization}	\label{sec:alg2}
Directly applying the BCD algorithm to the joint PMF recovery problem has a number of challenges. First, the CNMF problem is nonconvex, and thus global optimality is not ensured.
Second, the computational complexity grows quickly with $N$ and $I$, giving rise to potential scalability issues.
In this subsection, we offer a fast algorithm that can provably identify ${\sf Pr}(i_n|f)$ and ${\sf Pr}(f)$ up to a certain accuracy---under more stringent conditions relative to those in Theorem~\ref{prop:prop2}. The algorithm is reminiscent of a Gram--Schmidt-like algorithm for NMF with careful construction of a `virtual NMF model'.

Let us assume that there is a splitting ${\cal S}_1=\{\ell_1,\ldots,\ell_M\}$ and ${\cal S}_2=\{\ell_{M+1},\ldots,\ell_{N}\}$ such that the $\widetilde{\X}=\W\H^\T$ in \eqref{eq:constructedX} admits a full rank $\W$ and a separable $\H$.            
Under the separability condition, we have ${\bm H}(\bm \varLambda,:) =\bm \Sigma= {\rm Diag}(\alpha_1,\ldots,\alpha_F)$ and {${\bm W}\bm \Sigma = \widetilde{\X}(:,\bm \varLambda).$} Hence, the coupled NMF task boils down to identifying the index set $\bm \varLambda$.  The so-called \textit{successive projection algorithm} (SPA) from the NMF literature \cite{Gillis2012,MC01,fu2014self,arora2012practical} can be employed for this purpose. A remark is that SPA is a Gram--Schmidt-like algorithm, which only consists of norm comparison and orthogonal projection. 

Once ${\bm W}$ is identified, one can recover $\bm A_{\ell_n}\in \mathbb{R}^{I_{\ell_n} \times F}$ for $\ell_n\in{\cal S}_1$ up to identical column permutations, by extracting the corresponding rows of $\bm W$ (cf. line 6 in Algorithm~\ref{algo:CNMF-SPA}).
Unlike general NMF models, since every column of $\A_n$ is a conditional PMF, there is no scaling ambiguity.
The $\H$ matrix can be estimated using (constrained) least squares, and $\A_{\ell_n}$ for $\ell_n\in {\cal S}_2$ can then be extracted in a similar way. 
Denote \eqref{eq:constructedX} as $\widetilde{\bm X} = {\bm W} \bm D(\bm \lambda) \widetilde{\bm H}^{\top}$, where $\widetilde{\bm H} = \begin{bmatrix} \bm A_{\ell_{M+1}}^{\top} \ldots \bm A_{\ell_N}^{\top} \end{bmatrix}^\T$. 
Then, the PMF of the latent variable can be estimated via
$
{ \bm \lambda} = ( \widetilde{\bm H} \odot \W)^{\dagger}\text{vec}(\widetilde{\bm X}),$
%=\bm \Pi\bm \lambda.$
where we have used the fact that the Khatri-Rao product $\widetilde{\bm H} \odot \W$ has full column rank since both ${\W}$ and $\widetilde\H$ have full column rank. 

A remark is that the SPA-estimated $\widehat{\bm A}_{\ell_n}$ is at best the column-permuted version of $\A_{\ell_n}$ \cite{Gillis2012}.
However, since the permutation ambiguity across all the $\bm A_n$'s and the $\bm \lambda$ are identical per the above procedure, the existence of column permutations does not affect the ``assembling'' of ${\sf Pr}(i_n|f)$ and ${\sf Pr}(f)$ to recover ${\sf Pr}(i_1,\ldots,i_N)$.
We refer to this procedure as {\it coupled NMF via SPA} ({CNMF-SPA}); see Algorithm \ref{algo:CNMF-SPA}. 

From Algorithm~\ref{algo:CNMF-SPA}, one can see that the procedure is rather simple since SPA is a Gram--Schmidt-like procedure and the other steps are either convex quadratic programming or least squares.
Hence, {CNMF-SPA} can be employed for initializing computationally more intensive algorithms, e.g., {CNMF-OPT}.

	\begin{algorithm}[h]
	\footnotesize
	\SetKwInOut{Input}{input}
	\SetKwInOut{Output}{output}
	\SetKwRepeat{Repeat}{for}{end}
	%\SetCustomAlgoRuledWidth{10cm} 
	%\caption{{CNMF-SPA}}\label{algo:CNMF-SPA}
	%\begin{algorithmic}
	\Input{data samples $\{\bm d_s\}_{s=1}^S$ and $M$}
	 estimate second order statistics $\widehat{\X}_{jk }$;
	 
    split $\{1,\ldots,N\}$ into ${\cal S}_1=\{1,\ldots,M\}$ and ${\cal S}_2=\{M+1,\ldots,N\}$;
	 
	 Construct $\widetilde{\bm X}$;% 
	 
	 Estimate $\widehat{\bm W}$ using the {\bf SPA algorithm} \cite{Gillis2012} to select $\bm \varLambda$;
	 
	 {$ b \leftarrow 0$;}
	 
	\For{$n=1$ {\bf to} $M$}{
	 {$\widehat{\bm A}_n \leftarrow \widehat{\bm W} (b+[1,\dots,I_n],:)$;}
	 
	 normalize columns of $\widehat{\bm A}_n$ with respect to $\ell_1$ norm;
	 
	 {$b \leftarrow b+I_n$;}}
	 
 $\widehat{\bm H} \leftarrow \underset{{\bm H} \ge 0}{\text{arg~min}}~\|\widetilde{\bm X}- \widehat{\bm W} {\bm H}^{\T}\|_{\rm F}^2$;% or $\widehat{\bm H}=\widehat{\bm W}^\dagger\widetilde{\bm X}$;
 
 {$ b \leftarrow 0$;}
 
	\For{$n=M+1$ {\bf to} $N$}{		
	 {$\widehat{\bm A}_n \leftarrow \widehat{\bm H} (b+[1,\dots,I_n],:)$;}
		
		 normalize columns of $\widehat{\bm A}_n$ with respect to $\ell_1$ norm;
		 
		{$b \leftarrow b+I_n$;}
		 }
	
	 $\widetilde{\bm W}^{\T} \leftarrow \begin{bmatrix} \widehat{\bm A}_1^{\T}, \hdots, \widehat{\bm A}_M^{\T} \end{bmatrix}$ ;
	
	$\widetilde{\bm H}^{\T} \leftarrow  \begin{bmatrix} \widehat{\bm A}_{M+1}^{\top} \hdots \widehat{\bm A}_{N}^{\top} \end{bmatrix}^{\top} $;
	
 $\widehat{\bm \lambda} \leftarrow ( \widetilde{\bm H} \odot \widetilde{\bm W})^{\dagger}\text{vec}(\widetilde{\bm X})$; 	
	
	 \Output{estimates $\{\widehat{\A}_n\}_{n=1}^N$, $\widehat{\bm \lambda}$.}
	 \caption{{CNMF-SPA}}\label{algo:CNMF-SPA} 

%	\end{algorithmic}
\end{algorithm}
\vspace{-.25cm}

\subsection{Performance Analysis of CNMF-SPA} \label{sec:analysisSPA}
In principle, to make {CNMF-SPA} work, one needs to identify a splitting scheme ${\cal S}_1$ and ${\cal S}_2$ such that $\bm H$ in \eqref{eq:constructedX} satisfies the separability condition. 
It is impossible to know such a splitting {\it a priori}.
Testing all combinations of ${\cal S}_1$ and ${\cal S}_2$ gives rise to a hard combinatorial problem.
In addition, the pairwise marginals $\X_{jk}$ are not perfectly estimated in practice. %Instead, they are estimated using a finite number of samples (and often with missing observations), thereby being noisy.
%If there are many missing values in the data samples, the estimated pairwise marginals becomes even more inaccurate. 
In this subsection, we offer performance analysis of {CNMF-SPA} by taking the above aspects into consideration.

From our extensive experiments, we observed that using the `naive' construction ${\cal S}_1=\{1,\ldots,M\}$ and ${\cal S}_2=\{M+1,\ldots,N\}$ seems to work reasonably well for various tasks.
To establish theoretical understanding to such effectiveness, recall that we use $S$ and $p$ to denote the number of available realizations of ${\sf Pr}(Z_1,\ldots,Z_N)$ and the probability of observing each entry of data sample $\bm d_s$, respectively.   
For simplicity, we assume that $I_n=I$ for all $n$.
We impose the following generative model for $\A_n$:

\begin{Assumption}\label{ass:A}
	Assume that every row of $\A_n\in\mathbb{R}^{I_m\times F}$ for all $n$ is generated from the $(F-1)$-probability simplex uniformly at random and then positively scaled by a scalar, so that $\bm 1^\T\bm A_n=\bm 1^\T$ is respected\footnote{Such scaling exists under mild conditions; see \cite[Proposition 1]{yang2019learning}.}.
\end{Assumption}
\color{black}
Under the above settings, we show that the following holds:
\begin{theorem} \label{thm:spabound}
	Assume that {$\|\widehat{\bm X}_{{jk}}(:,q)\|_1 \ge \eta>0$ for any $q,{j,k}$}.
	Also, assume that
	$M \ge  F/I $, $p \geq \left( \frac{8}{S}\log(4/\delta) \right)^{1/2}$,
	\begin{align*}
	S & = \Omega\left(\frac{M^2I{\rm log}(1/\delta)}{\sigma^2_{\rm max}({\bm{W}})\eta^2\varepsilon^2p^2}\right),\\
	N &= M+ \Omega\left(\frac{\varepsilon^{-2(F-1)}}{IF}{\rm log}\left(\frac{F}{\delta}\right)\right), 	    
	\end{align*}
	where $ {0 < \varepsilon \le \frac{M\text{\rm min}\left(\frac{1}{2\sqrt{F-1}},\frac{1}{4}\right)}{2\kappa(\bm{W})(1+80\kappa^2(\bm{W}))}}$.
	Then,   under the defined $S,p$ and Assumption~\ref{ass:A}, {CNMF-SPA} outputs $\widehat{\A}_m$'s such that
	\begin{align} \label{eq:thmspa}
	\min_{\bm \Pi:~\text{permuation}}\|\widehat\A_m\bm \Pi - \A_m\|_2 = O\left(\kappa^2(\bm W)\sqrt{F} \zeta\right)
	\end{align}
	for $m\in{\cal S}_1$ with a probability greater than or equal to $1-\delta$, where $\zeta = \max(\sigma_{\max}(\bm W)\varepsilon,\nicefrac{M\sqrt{I\log(1/\delta)}}{\eta p \sqrt{S}})$.
\end{theorem}
The proof is relegated to Appendix~\ref{app:spa}. 
Theorem~\ref{thm:spabound} asserts that using {CNMF-SPA} on the constructed $\widetilde{\X}$ in \eqref{eq:constructedX} approximately recovers the joint PMF---if the $N,I,S$ are large enough. 
Note that if $\A_m$ for all $m \in{\cal S}_1$ can be accurately estimated, the estimation accuracy of
$\bm A_n$ for all $n \in{\cal S}_2$ and $\bm \lambda$ can also be guaranteed and quantified, following the standard sensitivity analyses of least squares; see, e.g., \cite{wedin1973perturbation}. %We omit this part for conciseness.

\begin{remark} \label{rem:theorem_spa}
Theorem~\ref{thm:spabound} is not entirely surprising.
The insight behind the proof is to recast the constructed $\widetilde{\X}$ as an equivalent noisy NMF model with the separability condition holding exactly on the right latent factor. 
This way, the noise robustness analysis for SPA \cite{Gillis2012} can be applied.
The key is to quantify the noise bound in this equivalent model. Such `virtual noise' is contributed by the finite sample-induced error for estimating the pairwise marginals and the violation to the separability condition.
In particular, 
classical concentration theorems are leveraged to quantify the former, and Assumption~\ref{ass:A} for the latter.
Note that Assumption~\ref{ass:A} is more of a working assumption, rather than an exact model for capturing reality---it helps formalize the fact that $\H$ in \eqref{eq:constructedX} more likely attains the $\varepsilon$-separability condition when $|{\cal S}_2|$ grows \cite{ibrahim2019crowdsourcing}. In principle, if $\A_m$ is drawn from any joint absolutely continuous distribution, a similar conclusion can be reached---but using the uniform distribution simplifies the analysis.
\end{remark}

\color{black}

\subsection{More Discussions: EM Meets CNMF}
As mentioned, a recent work in \cite{yeredor2019maximum} proposed an EM algorithm for solving the maximum likelihood estimation problem under \eqref{eq:pmf_latent_var}. 
To see the idea, let $f_s \in \{1,\dots,F\}$ be the realization of the `latent variable' $H$ in the $s$th realization of ${\sf Pr}(Z_1,\ldots,Z_N)$. The joint log-likelihood of the observed data $\{\bm d_s\}_{s=1}^S$ and the corresponding latent variable $\{f_s\}_{s=1}^S$ as a function of the latent factor $\bm \theta =\begin{bmatrix}{\rm vec}(\bm A_1)^\T,\dots,{\rm vec}(\bm A_N)^\T, \bm \lambda^{\T}\end{bmatrix}^\T$ can be expressed as 
\begin{align}\label{eq:loglike}
{\cal L}&(\{\bm d_s,f_s\}_{s=1}^S;\bm \theta)\\
&=\log\left(\prod_{s=1}^S \bm \lambda(f_s) \prod_{n=1}^N\prod_{i=1}^{I_n} \bm A_n(i,f_s)^{\mathbb{I}(\bm d_s(n)=z_{n}^{(i)})}\right).\nonumber
\end{align}
The EM algorithm estimates $\bm \theta$ via attempting to maximize the log-likelihood. The resulting updates are simple and economical. However, the optimality of the algorithm is unclear. It was also observed in \cite{yeredor2019maximum} that the algorithm sometimes converges to some undesired solutions if not carefully initialized.

Naturally, our idea is using {CNMF-SPA} to initialize the EM algorithm in \cite{yeredor2019maximum}. As one will see, this combination oftentimes offers competitive performance.%---after all, the ML estimation is a `gold standard' for statistical learning.
This combination also features an appealing runtime performance, since both algorithms consist of lightweight updates. We refer to this procedure as CNMF-SPA-initialized EM ({CNMF-SPA-EM}).

%The EM algorithm is for solving the MLE for the naive Bayes model in \eqref{eq:pmf_latent_var}, which is a nonconvex optimization problem. 
%It is unclear if the EM algorithm converges to reasonable estimates for ${\sf Pr}(i_n|f)$ and ${\sf Pr}(f)$.

In this work, we also present performance characterizations for the EM algorithm.
To proceed, we make the following assumption:

\begin{Assumption}\label{ass:D}
	Define 
	$\overline{D}_1 = \min_{f \neq f'}\frac{1}{N}\sum_{n=1}^N p\mathbb{D}_{\rm KL}(\bm A_n(:,f),\bm A_n(:,f'))$, $
	\overline{D}_2 = \frac{2}{N}\min_{f\neq f'} \log(\bm \lambda(f)/\bm \lambda(f'))$
	and $\overline{D}=\nicefrac{(\overline{D}_1 + \overline{D}_2)}{2}$.  
	Assume that  ${\bm A}_n(i,f) \ge \rho_1$ and ${\bm \lambda}(f) \ge \rho_2$ for all $n,i,f$, and that $\bm A_n, \bm \lambda$ and the initial estimates $	\widehat{\bm A}^0_n, \widehat{\bm \lambda}^0$ satisfy
	\begin{align*}
	    |\widehat{\bm A}^0_n(i,f)-{\bm A}_n(i,f)| &\le \delta_1:=\frac{4}{\rho_1(4+ \overline{D})}\\
	    |\widehat{\bm \lambda^0}(f)-{\bm \lambda}(f)|& \le \delta_2:= \frac{4}{\rho_2(4+ N\overline{D})}.
	\end{align*}
\end{Assumption}
Here, $\overline{D}_1$ can be understood as a measure for the `conditioning' of $\A_n$ under the KL divergence---a larger $\overline{D}_1$ implies more diverse columns of $\A_n$, and thus a better `condition number'.
In addition, $|\overline{D}_2|$ measures how far $\bm \lambda$ is away from the uniform distribution. 
Using Assumption~\ref{ass:D}, we show that the EM algorithm can provably improve upon reasonable initial guesses (e.g., those given by {CNMF-SPA}), under some conditions.
To be specific, we have the following theorem:
\begin{theorem} \label{thm:em}
	Let $\delta_{\min} = \min(\delta_1,\delta_2)$. 
	Assume that the following hold:
	\begin{align*}
	N &\ge \max\left(\frac{33\log(3SF/\mu)}{\rho_1\overline{D}_1},\frac{4\log(4SF^2/(3p\rho_2\mu))}{\overline{D}}\right),\\  
	S &\ge \frac{192F^2\log(12NFI/\mu)}{p^2\rho_2^2\delta^2_{\min}},~\overline{D} \ge \max\left\{\frac{8-4\rho_1^2}{\rho_1^2},\frac{8-4\rho_2^2}{N\rho_2^2}\right\}.
	\end{align*}
	Then, under Assumption~\ref{ass:D},
	the EM algorithm in \cite{yeredor2019maximum} outputs $\widehat{\bm A}_n(i,f),\widehat{\bm \lambda}(f)$ that satisfy the following with a probability greater than or equal to $1-\mu$:
	\begin{align*} 
	&	|\widehat{\bm A}_n(i,f)-\bm A_n(i,f)|^2 \le  \frac{48\log(12NFI/\mu)}{Sp\bm \lambda(f)}\leq \delta_1^2 ,\\
	&|\widehat{\bm \lambda}(f)-\bm \lambda(f)|^2 \le \frac{192F^2\bm \lambda(f)\log(12NFI/\mu)}{S}\leq \delta_2^2.
	\end{align*}
\end{theorem}
The proof is relegated to Appendix~\ref{app:thm_em}.
Our proof extends the analysis of a different EM algorithm proposed in \cite{zhang2014spectral} that is designed for learning the Dawid-Skene model in crowdsourcing. The EM algorithm there effectively learns a naive Bayes model when the latent variable admits a uniform distribution, i.e., $\bm \lambda(f)=1/F$ for all $f$.
The algorithm in \cite{zhang2014spectral} and the EM algorithm in \cite{yeredor2019maximum} are closely related, but the analysis in \cite{zhang2014spectral} does not cover the latter---{since the EM in \cite{yeredor2019maximum} does not restrict} $\bm \lambda$ to be a uniform PMF.
Extending the analysis in \cite{zhang2014spectral} turns out to be nontrivial, which requires careful characterization for the `difficulty' introduced by $\bm \lambda$'s deviation from the uniform distribution. 
Our definition of $\overline{D}_2$ and Theorem~\ref{thm:em} fill the gap.

\section{Experiments}
In this section, we use synthetic and real data experiments to showcase the effectiveness of the proposed algorithms.
\subsection{Baselines}\label{sec:baselines}
We use the CTD based method in \cite{kargas2017tensors} and the MLE-based method in \cite{yeredor2019maximum} as the major benchmarks, since they are general joint PMF estimation approaches as the proposed methods. Since the MLE-based EM algorithm in \cite{yeredor2019maximum} may be sensitive to initialization, an iterative optimization algorithm designed for the same MLE formulation (cf. the alternating directions (AD) algorithm in \cite{yeredor2019maximum}) is used to initialize the EM in \cite{yeredor2019maximum}. This baseline is denoted as the {MLE-AD-EM} method in our experiments.

For the real data experiments, we test the proposed approaches on two core tasks in machine learning, namely, classification and recommender systems.
For the classification task, we use 4 different classifiers from the MATLAB
\textit{Statistics and Machine Learning Toolbox}, i.e., linear SVM, kernel SVM with radial basis function (RBF)), multinomial logistic regression and naive Bayes classifier. 
For the user-item recommender systems, we use the celebrated \textit{biased matrix factorization} (BMF) method \cite{koren2009matrix} as a baseline. We also
compare with results obtained by \textit{global average} of the ratings, the
\textit{user average}, and the \textit{item average} for predicting
the missing entries in the user-item matrix. 

All the algorithms are implemented in MATLAB 2018b and run on a desktop machine equipped with i7 3.40 GHZ CPU. 
The iterative algorithms used in the experiments are stopped when the relative change in the objective function value is less than $10^{-6}$.

\begin{table}[t!]
	\centering
	\caption{MSE \& MRE Performance on Synthetic Data with $N=5,F=5,I=10,p=0.5,\varepsilon=0.1$.}
	\resizebox{.7\linewidth}{!}{\Huge
		\begin{tabular}{|c|c|c|c|c|c|}
			\hline
			\textbf{Algorithms} & \textbf{Metric} & $S=10^3$ & $S=10^4$ & $S=10^5$ & $S=10^6$ \\
			\hline
			CNMF-SPA [Proposed] & MSE & 0.0669 & 0.0270 & 0.0210 & \textbf{0.0204}  \\
			\hline
			CNMF-OPT [Proposed] & MSE & \textbf{0.0519} & \textbf{0.0242} & 0.0210 & \textbf{0.0204} \\
			\hline
			CNMF-SPA-EM [Proposed] & MSE & 0.0553 & 0.0247 & \textbf{0.0208} & \textbf{0.0204} \\
			\hline
			RAND-EM & MSE & 0.0734 & 0.0317 & 0.0289 & 0.0336 \\
			\hline
			CTD[Kargas \textit{et al.}]   & MSE &  0.1523 & 0.0256 & 0.0212 & 0.0205 \\
			\hline
			%MLE-AD [Yeredor \& Haardt]  & MSE & 0.1215 & 0.0879 & 0.0746 & 0.0673 \\
			%\hline 
			MLE-AD-EM [Yeredor \& Haardt] & MSE & 0.0909 & 0.0368 & 0.0266 & 0.0296 \\
			\hline
			\hline
			CNMF-SPA [Proposed] & MRE & 0.7965 & 0.3321 & 0.1052 & 0.0346  \\
			\hline
			CNMF-OPT [Proposed] & MRE & \textbf{0.6746} & \textbf{0.2327} & 0.0721 & 0.0239 \\
			\hline
			CNMF-SPA-EM [Proposed] & MRE & 0.6788 & 0.2341 & \textbf{0.0664} & \textbf{0.0219} \\
			\hline
			RAND-EM & MRE & 0.8078 & 0.2820 & 0.1845 & 0.2610 \\
			\hline
			CTD [Kargas \textit{et al.}] & MRE  & 0.9076 & 0.2920 & 0.0952 & 0.0329 \\
			\hline
			%MLE-AD [Yeredor \& Haardt]  & MRE & 1.2810 & 1.4311 & 1.4277 & 1.3347 \\
			%\hline 
			MLE-AD-EM [Yeredor \& Haardt] & MRE & 0.8363 & 0.3552 & 0.1384 & 0.1593 \\
			\hline
		\end{tabular}%
	}
	\vspace{-.25cm}
	\label{tab:synthetic2}%
\end{table}%

\begin{table}[t!]
	\centering
	\caption{MSE \& MRE Performance on Synthetic Data with $N=5,F=5,I=10,p=0.5,\varepsilon=0.3$.}
	\resizebox{.7\linewidth}{!}{\Huge
		\begin{tabular}{|c|c|c|c|c|c|}
			\hline
			\textbf{Algorithms} & \textbf{Metric} & $S=10^3$ & $S=10^4$ & $S=10^5$ & $S=10^6$ \\
			\hline
			CNMF-SPA [Proposed] & MSE & 0.0880 & 0.0288 & 0.0238 & 0.0237 \\
			\hline
			CNMF-OPT [Proposed] & MSE & \textbf{0.0810} & \textbf{0.0247} & \textbf{0.0204} & 0.0210 \\
			\hline
			CNMF-SPA-EM [Proposed] & MSE & 0.0813 & 0.0249 & 0.0206 & 0.0217 \\
			\hline
			RAND-EM & MSE & 0.0883 & 0.0292 & 0.0216 & 0.0276 \\
			\hline
			CTD[Kargas \textit{et al.}]   & MSE   & 0.1534 & 0.0266 & \textbf{0.0204} & \textbf{0.0206} \\
			\hline
			%MLE-AD [Yeredor \& Haardt]  & MSE & 0.1153 & 0.0768 & 0.0732 & 0.0678 \\
			%\hline 
			MLE-AD-EM [Yeredor \& Haardt] & MSE & 0.0911 & 0.0312 & 0.0313 & 0.0233 \\
			\hline
			\hline
			CNMF-SPA [Proposed] & MRE & 0.8933 & 0.4395 & 0.3201 & 0.3128 \\
			\hline
			CNMF-OPT [Proposed] & MRE & \textbf{0.8010} & 0.2675 & 0.1021 & 0.0784 \\
			\hline
			CNMF-SPA-EM [Proposed] & MRE & 0.8133 & \textbf{0.2658} & 0.1132 & 0.1151 \\
			\hline
			RAND-EM & MRE & 0.8214 & 0.3080 & 0.1390 & 0.2239 \\
			\hline
			CTD [Kargas \textit{et al.}] & MRE  & 0.9101 & 0.3298 & \textbf{0.1017} & \textbf{0.0346} \\
			\hline
			%MLE-AD [Yeredor \& Haardt]  & MRE & 1.2215 & 1.3440 & 1.2948 & 1.3080 \\
			%\hline 
			MLE-AD-EM [Yeredor \& Haardt] & MRE & 0.8752 & 0.3593 & 0.2666 & 0.1413 \\
			\hline
		\end{tabular}%
	}
	\vspace{-.25cm}
	\label{tab:synthetic_ep3}%
\end{table}%

\subsection{Synthetic-data Experiments} \label{sec:synthetic}
We consider $N=5$ RVs where each variable takes $I=10$ discrete values. The entries of the conditional PMF matrices (factor matrices) ${\bm A}_n\in \mathbb{R}^{I_n \times F}$ and the prior probability vector ${\bm \lambda} \in \mathbb{R}^{F}$ are randomly generated from a uniform distribution between 0 and 1 and the columns are rescaled to have unit $\ell_1$ norm with rank $F=5$. 
We generate $S$ realizations of the joint PMF by randomly hiding each variable realization with observation probability $p=0.5$.
In particular, to test the performance of CNMF-SPA under various $\varepsilon$-separability conditions, we fix $M=3$ and manually ensure the $\varepsilon$-separability condition on $\bm H$ in \eqref{eq:constructedX} by fixing $\varepsilon \in\{0.1,0.3\}$. 
To run the proposed algorithms, the low-order marginals are then estimated from the realizations via sample averaging. The {\it mean square error} (MSE) of the factors (see~\cite{huang2014non} for definition) and the {\it mean relative error} (MRE) of the recovered joint PMFs (see \cite{kargas2017tensors}) are evaluated. 
We should mention that, ideally, MRE is more preferred for evaluation, but it is hard to compute (due to memory issues) for large $N$. MSE is an intermediate metric since it does not directly measure the recovery performance. Low MSEs imply good joint PMF recovery performance, but the converse is not necessarily true.
The results are averaged from 20 random trials.

Table \ref{tab:synthetic2}-\ref{tab:synthetic_ep3} present the MSE and MRE results for various values of $\varepsilon$ and $S$, where we manually control $\varepsilon$ when generating the $\A_n$'s. One can see that the proposed CNMF-SPA method works reasonably well for both $\varepsilon$'s under test.
It performs better when $\varepsilon$ is small, corroborating our analysis in Sec~\ref{sec:analysisSPA}. In addition, CNMF-SPA is quite effective for initializing the EM algorithm whereas randomly initialized EM (RAND-EM) and the MLE-AD-EM method from \cite{yeredor2019maximum} sometimes struggle to attain good performance. Notably, the proposed pairwise marginals based methods consistently outperform the three-dimensional marginals based method, i.e., CTD from \cite{kargas2017tensors}, especially when $S$ is small. This highlights the sample complexity benefits of the proposed methods. Another point is that when CNMF-SPA does not result in good accuracy, then the follow-up EM stage (i.e., CNMF-SPA-EM) does not perform well (e.g., when $S=10^3$). However, the proposed coupled factorization based approach, i.e., CNMF-OPT, still works well in this `sample-starved regime', showing the power of recoverability guaranteed problem formulation. 

\begin{table}[t]
  \centering
  {\caption{MSE Performance on Synthetic Data with $N=5, F=5, I=10, S =100000, \varepsilon = 0.1$.}
  	\resizebox{.7\linewidth}{!}{\Huge
    \begin{tabular}{|c|c|c|c|c|c|}
    \hline
    \textbf{Algorithms} & \textbf{Metric} & $p=0.1$ & $p=0.05$ & $p=0.01$ & $p=0.005$ \\
    \hline
    CNMF-SPA [Proposed] & MSE   & 0.0370 & 0.1132 & 0.4674 & 0.4948 \\
    \hline
    CNMF-OPT [Proposed] & MSE   & \textbf{0.0257} & \textbf{0.0564} & 0.3992 & 0.7481 \\
    \hline
    CNMF-SPA-EM [Proposed] & MSE   & 0.0311 & 0.0816 & 0.2908 & 0.2046 \\
    \hline
    RAND-EM & MSE   & 0.0703 & 0.0923 & 0.0978 & 0.1014 \\
    \hline
    CTD [Kargas \textit{et al.}]   & MSE   & 0.0606 & 0.2759 & 0.4459 & 0.6922 \\
    \hline
    MLE-AD-EM [Yeredor \& Haardt] & MSE   & 0.0697 & 0.1094 & \textbf{0.0824} & \textbf{0.0862} \\
    %\hline
    % \hline
    % CNMF-SPA [Proposed] & MRE   & 0.5398 & 0.9444 & 1.0024 & 1.0101 \\
    % \hline
    % CNMF-OPT [Proposed] & MRE   & 0.3438 & 0.7142 & 1.0009 & 1.0002 \\
    % \hline
    % CNMF-SPA-EM [Proposed] & MRE   & 0.3920 & 0.8509 & 1.0085 & 1.0482 \\
    % \hline
    % RAND-EM & MRE   & 1.2499 & 1.3381 & 1.4376 & 1.3995 \\
    % \hline
    % CTD   & MRE   & 0.6645 & 1.0450 & 1.0056 & 1.0003 \\
    % \hline
    % MLE-AD-EM & MRE   & 0.9184 & 1.1953 & 1.5817 & 1.5049 \\
    \hline
    \end{tabular}%
    }
  \label{tab:mse_smallp}%
  }
\end{table}%

{
It should be noted that marginal distribution based methods' performance relies on the estimation accuracy of the marginal PMFs.
Table \ref{tab:mse_smallp} presents the MSEs of the algorithms under various {\it very} small $p$'s, with $N=5, F=5, I=10, S =100,000$ and  $\varepsilon = 0.1$. 
One can see that when $p\leq 0.05$, the triple marginal based method (i.e., the CTD method) does not produce promising results. This is because triples are rarely observed when $p=0.05$ and thus $\widehat{\tX}_{jk\ell}$'s are fairly noisy.
The proposed method works well even when $p=0.05$, again showing the sample complexity advantages of using pairwise marginals.
When $p\leq 0.01$, both CTD and the proposed approach cannot accurately estimate the marginal distributions.
In such cases, the EM algorithm that does not use higher-order co-occurrences shows its advantages.

}

Table~\ref{tab:synthetic3}-\ref{tab:synthetic4} show the MSE results of the algorithms tested on cases where $N=15,I=10$ and $F=10$ for various values of $p$ and $S$---recall that a small $p$ means the data has many missing observations. 
The MRE is not presented since instantiating $15$th-order tensors is prohibitive in terms of memory.
The results for MLE-AD-EM is not presented since the AD stage consumes a large amount of time under this high-dimensional PMF setting. 
We do not manually impose $\varepsilon$-separability here, with the hope that the $\bm H$ matrix is likely $\varepsilon$-separable under large $N$ and $I$, as argued in Theorem \ref{thm:spabound} and Remark~\ref{rem:theorem_spa}. 
One can see that the proposed CNMF-based methods exhibit better MSE performance compared to other baselines.
Also, the advantage of CNMF-OPT is more articulated under challenging settings (e.g., small $p$ and small $S$), where one can see that CNMF-OPT almost always outperforms the other algorithms. 
% {\blue One remark is that the joint PMF estimation may not be successful using the marginals based methods CNMF and CTD when $p$ becomes extremely small. This is because these methods rely on reasonably well-estimated marginals which are hard to acquire with extremely small $p$. In such cases, other approaches such as MLE-based methods may have a better chance (see Appendix \ref{app:exp} for more details)}.

Table~\ref{tab:runtime} shows the average runtime of the algorithms under different problem sizes. One can see that for the smaller case (cf. the left column), all the methods exhibit good runtime performance except for MLE-AD-EM. The elongated runtime of MLE-AD-EM is due to the computationally expensive AD iterations used for initialization. For the high-dimensional case where $N=15,F=I=10$, CNMF-SPA and EM are much faster due to their computationally economical updates.
Considering the MSE and MRE results in Tables \ref{tab:synthetic2}-\ref{tab:synthetic4}, CNMF-SPA and CNMF-SPA-EM seem to provide a preferred good balance between speed and accuracy.  
Furthermore, the proposed CNMF-OPT offers the best accuracy for sample-starved cases---at the expense of higher computational costs.

\begin{comment}
We have presented the simulation results with $\varepsilon =0.1$ in Table \ref{tab:synthetic2} in Sec.~\ref{sec:synthetic} (the parameter $\varepsilon$ which is associated with separability condition is defined in Def. \ref{def:sep}). Here, we present the results by further increasing $\varepsilon$  to show the effect of this parameter on the performance of the proposed methods. The results are presented in Tables~\ref{tab:synthetic_ep2} and \ref{tab:synthetic_ep3}. One can notice that as $\varepsilon$ increases, the estimation error for CNMF-SPA increases as expected. However the proposed algorithms CNMF-OPT and CNMF-SPA still outperform the baselines especially when the number of samples is not very large ($S \le 10^5)$.
\end{comment}

\begin{table}[t!]
	\centering
	\caption{MSE Performance on Synthetic Data with $N=15,F=10,I=10,p=0.5$.}
	\resizebox{.7\linewidth}{!}{\Huge
		\begin{tabular}{|c|c|c|c|c|c|}
			\hline
			\textbf{Algorithms} & \textbf{Metric} & $S=10^3$ & $S=10^4$ & $S=10^5$ & $S=10^6$ \\
			\hline
			CNMF-SPA [Proposed] & MSE & 0.1183 & 0.1030 & 0.1063 & 0.1041 \\
			\hline
			CNMF-OPT [Proposed] & MSE & \textbf{0.0218} & \textbf{0.0042} & \textbf{0.0022} & 0.0020 \\
			\hline
			CNMF-SPA-EM [Proposed] & MSE & 0.0894 & 0.0110 & 0.0056 & \textbf{0.0018} \\
			\hline
			RAND-EM & MSE & 0.0376 & 0.0112 & 0.0149 & 0.0069 \\
			\hline
			CTD[Kargas \textit{et al.}]   & MSE &  0.0329 & 0.0359 & 0.0404 & 0.0355 \\
			\hline
		\end{tabular}%
	}
	\vspace{-0.25cm}
	\label{tab:synthetic3}%
\end{table}%

\begin{table}[t!]
	\centering
	\caption{MSE Performance on Synthetic Data for $N=15,F=10,I=10,p=0.2$.}
	\resizebox{.7\linewidth}{!}{\large
		\begin{tabular}{|c|c|c|c|c|c|}
			\hline
			\textbf{Algorithms} & \textbf{Metric} & $S=10^3$ & $S=10^4$ & $S=10^5$ & $S=10^6$ \\
			\hline
			CNMF-SPA [Proposed] & MSE & 0.2884 & 0.1181 & 0.0987 & 0.0996 \\
			\hline
			CNMF-OPT [Proposed] & MSE & 0.2277 & \textbf{0.0423} & \textbf{0.0062} & \textbf{0.0021} \\
			\hline
			CNMF-SPA-EM [Proposed] & MSE & \textbf{0.2274} & 0.0672 & 0.0165 & 0.0070 \\
			\hline
			RAND-EM & MSE & 0.2333 & 0.0782 & 0.0349 & 0.0313 \\
			\hline
			CTD[Kargas \textit{et al.}]   & MSE &  0.4015 & 0.0862 & 0.0081 & 0.0260 \\
			\hline
		\end{tabular}%
	}
	\label{tab:synthetic4}%
	\vspace{-0.25cm}
\end{table}%

\begin{table}[t!]
\centering
   \begin{threeparttable}[b]
	
	\caption{Average Runtime (Sec.) of the Algorithms on Synthetic Data $S=10^5$ under Various $[N,F,I]$'s.}
		\begin{tabular}{|c|c|c|}
			\hline
			\textbf{Algorithms} & $[5,5,10]$ & $[15,10,10]$ \\
			\hline
			CNMF-SPA [Proposed] & \textbf{0.139} & \textbf{0.865}  \\
			\hline
			CNMF-OPT [Proposed] & 1.633 & 94.931 \\
			\hline
			CNMF-SPA-EM [Proposed] & 0.490 & 4.091 \\
			\hline
			RAND-EM & 0.353 & 3.217 \\
			\hline
			CTD [Kargas \textit{et al.}]   & 0.741 & 50.559 \\
			\hline
			%MLE-AD [Yeredor \& Haardt] & 124.101 & $\approx 50$ minutes\\
			%\hline
			MLE-AD-EM [Yeredor \& Haardt] & 124.449 & $ 12.949\times 10^3~^\dagger$ \\
			\hline
		\end{tabular}%

	\label{tab:runtime}%\\

	      \begin{tablenotes}
     \item[$\dagger$] {\footnotesize This number is reported from a single trial due to the elongated runtime of the algorithm under this high-dimensional setting.}
   \end{tablenotes}
  \end{threeparttable}
	\vspace{-.25cm}
\end{table}%

\subsection{Real-data Experiments: Classification Tasks}	
We conduct classification tasks on a number of UCI datasets (see \url{https://archive.ics.uci.edu}) using the joint PMF recovery algorithms and dedicated data classifiers as we detailed in Sec.~\ref{sec:baselines}.  The details of the datasets are given in Table~\ref{tab:uci}. Note that in the datasets, some of the features are continuous-valued. To make the joint PMF recovery methods applicable, we follow the setup in \cite{kargas2017tensors} to discretize those continuous features. We use $I_{\sf ave}=(1/N)\sum_{n=1}^N I_n$ to denote the averaged alphabet size of the discretized features for each dataset in Table~\ref{tab:uci} \footnote{ The detailed setup and source code are available at \url{https://github.com/shahanaibrahimosu/joint-probability-estimation}.}.  
We split each dataset into training, validation, and testing sets with a ratio of $50:20:30$. 
For our approach, we estimate the joint PMF of the features and the label using the training set, and then predict the labels on the testing data by constructing an MAP predictor following \cite{kargas2017tensors}. The rank $F$ for the joint PMF recovery methods and the number of iterations needed for the methods involving EM are chosen using the validation set. We set $M=5$ for all datasets. We perform 20 trials for each dataset with randomly partitioned training/testing/validation sets, and report the mean and standard deviation of the classification accuracy.

Tables~\ref{tab:votes}-\ref{tab:mushroom} show the results for the UCI datasets `Votes', `Car', `Nursery' and `Mushroom', respectively. 
 In all cases, one can see that the proposed approaches CNMF-SPA, CNMF-OPT and CNMF-SPA-EM are promising in terms of offering competitive performance for data classification. 
 We should remark that our method is not designed for data classification, but estimating the joint PMF of the features and the label. The good performance on classification clearly support our theoretical analysis and the usefulness of the proposed algorithms.
 
 A key observation is that the CPD method in \cite{kargas2017tensors} does not perform as well compared to the proposed pairwise marginals based methods, especially when the number of data samples is small (e.g., the `Votes' dataset in Table~\ref{tab:votes}).  This echos our motivation for this work---using pairwise marginals is advantageous over using third-order ones when fewer samples are available.
 In terms of runtime, the proposed CNMF-SPA and CNMF-SPA-EM methods are faster than other methods by several orders of magnitude.

\begin{table}[t]
	\centering
	\caption{Details of UCI datasets.}
	\resizebox{.6\linewidth}{!}{\Huge
		\begin{tabular}{|l|r|r|r|r|}
			\hline
			\textbf{Dataset} & \textbf{$\#$ Samples} & \textbf{$\#$ Features} & \textbf{$I_{\text{avg}}$}  & \textbf{$\#$ Classes} \\
			\hline
			Votes & 435 & 17    & 2    & 2 \\
			\hline
			Car   & 1728  & 7     & 4     & 4 \\
			\hline
			%Credit & 690   & 9     & 5     & 2 \\
			%\hline
			Nursery & 12960 & 8    & 4    & 4 \\
			\hline
			Mushroom & 8124  & 22    & 6     & 2 \\
			\hline
		\end{tabular}%
	}
	\label{tab:uci}%
		\vspace{-.25cm}	%\vspace{-.4cm}
\end{table}%

\begin{table}[t!]
	\centering
	
	\caption{Classification Results on UCI Dataset `Votes'.}
	\resizebox{.6\linewidth}{!}{\Huge
			\begin{tabular}{|c|c|c|}
			\hline
			\textbf{Algorithm} & \textbf{Accuracy (\%)} & \textbf{Time (sec.)} \\
			\hline
			CNMF-SPA [Proposed] & 90.07$\pm$1.30 & 0.005 \\
			\hline
			CNMF-OPT [Proposed] & \textbf{94.94$\pm$2.13} & 9.720 \\
			\hline
			CNMF-SPA-EM [Proposed] & 92.82$\pm$2.17 & 0.018 \\
			\hline
			CTD [Kargas \textit{et al.}] & 92.37$\pm$2.88 & 6.413 \\
			\hline
			%MLE-AD [Yeredor \& Haardt] &66.18$\pm$11.89 & 9.144\\
			%\hline
			MLE-AD-EM [Yeredor \& Haardt] &90.99$\pm$4.18 & 10.556\\
			\hline
			SVM   & 94.34$\pm$1.68 & 0.057 \\
			\hline
			Logistic Regression & 94.05$\pm$1.24 & 0.301 \\
			\hline
			%	Neural Net  & 92.82$\pm$3.28 & 0.174 \\
			%\hline
			SVM-RBF & 92.21$\pm$2.66 & 0.012 \\
			\hline
			Naive Bayes & 90.46$\pm$1.98 & 0.032 \\
			\hline
	\end{tabular}}
	\label{tab:votes}%
		\vspace{-.25cm}
\end{table}%

\begin{table}[t!]
	
	\centering
	\caption{Classification Results on UCI Dataset `Car'.}
	\resizebox{.6\linewidth}{!}{\Huge
				\begin{tabular}{|c|c|c|}
			\hline
			\textbf{Algorithm} & \textbf{Accuracy (\%)} & \textbf{Time (sec.)} \\
			\hline
			CNMF-SPA [Proposed] & 70.31$\pm$2.17  & 0.007 \\
			\hline
			CNMF-OPT [Proposed] & 85.00$\pm$1.80 & 5.550 \\
			\hline
			CNMF-SPA-EM [Proposed] & \textbf{87.42$\pm$1.12} & 0.017\\
			\hline
			CTD [Kargas \textit{et al.}] & 84.21$\pm$1.27 & 3.601 \\
			\hline
			%MLE-AD [Yeredor \& Haardt] &82.99$\pm$3.48 & 9.221\\
			%\hline
			MLE-AD-EM [Yeredor \& Haardt] &87.25$\pm$1.41 & 6.321\\
			\hline
			SVM   & 84.59$\pm$3.21 & 0.149 \\
			\hline
			Logistic Regression & 83.09$\pm$2.42 & 1.528 \\
			\hline
			%Neural Net  & 86.18$\pm$2.55 & 0.128 \\
			%\hline
			SVM-RBF & 77.13$\pm$4.29 & 0.876 \\
			\hline
			Naive Bayes & 83.32$\pm$2.23 & 0.023 \\
			\hline
		\end{tabular}%
	}
	\label{tab:uci_car}%
%	\vspace{-.5cm}
	\vspace{-.25cm}
\end{table}%
\begin{comment}
\begin{table}[t!]
	\centering
	\caption{Classification Results on UCI Dataset `Credit'}
	\resizebox{.8\linewidth}{!}{
		\begin{tabular}{|c|c|c|}
			\hline
			\textbf{Algorithm} & \textbf{Accuracy (\%)} & \textbf{Time (s)} \\
			\hline
			CNMF-SPA [Proposed] & 84.47$\pm$2.12 & 0.008 \\
			\hline
			\multicolumn{1}{|l|}{CNMF-OPT [Proposed]} & \textbf{84.92$\pm$2.19} & 9.945 \\
			\hline
			CNMF-SPA-EM [Proposed] & 84.64$\pm$1.83 & 0.018 \\
			\hline
			CTD [Kargas \textit{et al.}] & 84.42$\pm$2.31 & 3.374 \\
			\hline
			%MLE-AD [Yeredor \& Haardt] &83.65$\pm$2.41 & 23.324\\
			%\hline
			MLE-AD-EM [Yeredor \& Haardt] &84.52$\pm$1.90 & 23.334\\
			\hline
			SVM   & 84.18$\pm$2.04 & 0.026 \\
			\hline
			Linear Regression & 84.88$\pm$2.00 & 0.016 \\
			\hline
			Neural Net  & 84.20$\pm$3.20 & 0.107 \\
			\hline
			SVM-RBF & 81.11$\pm$2.25 & 0.009 \\
			\hline
			Naive Bayes & 84.17$\pm$2.07 & 0.016 \\
			\hline
		\end{tabular}%
	}
	\label{tab:uci_credit}%
%	\vspace{-.5cm}
\end{table}%
\end{comment}

\begin{table}[t!]

	\centering
	\caption{Classification Results on UCI Dataset `Nursery'.}
	\resizebox{.6\linewidth}{!}{\Huge
				\begin{tabular}{|c|c|c|}
			\hline
			\textbf{Algorithm} & \textbf{Accuracy (\%)} & \textbf{Time (sec.)} \\
			\hline
			CNMF-SPA [Proposed] & 97.48$\pm$0.20 & 0.008 \\
			\hline
			CNMF-OPT [Proposed] & \textbf{98.16$\pm$0.22} & 6.719 \\
			\hline
			CNMF-SPA-EM [Proposed] & 98.04$\pm$0.35 & 0.041 \\
			\hline
			CTD [Kargas \textit{et al.}] & 97.70$\pm$0.22 & 4.318 \\
			\hline
			%MLE-AD [Yeredor \& Haardt] &59.24$\pm$17.07 & 9.287\\
			%\hline
			MLE-AD-EM [Yeredor \& Haardt] &98.04$\pm$0.33 & 8.920\\
			\hline			
			SVM   & 97.42$\pm$0.79 & 0.706 \\
			\hline
			Logistic Regression & 98.02$\pm$0.15 & 5.019 \\
			\hline
			%Neural Net  & 97.96$\pm$0.30 & 0.387 \\
			%\hline
			SVM-RBF & 70.92$\pm$4.70 & 0.738 \\
			\hline
			Naive Bayes & 98.03$\pm$0.20 & 0.046 \\
			\hline
		\end{tabular}
	}
	\label{tab:nursery}%
	\vspace{-.25cm}
\end{table}%

\begin{table}[t!]
	\centering
	\caption{Classification Results on UCI Dataset `Mushroom'.}
	\resizebox{.6\linewidth}{!}{\Huge
		\begin{tabular}{|c|c|c|}
			\hline
			\textbf{Algorithm} & \textbf{Accuracy (\%)} & \textbf{Time (sec.)} \\
			\hline
			CNMF-SPA [Proposed] & 91.86$\pm$6.33 & 0.025 \\
			\hline
			CNMF-OPT [Proposed] & 96.70$\pm$0.82 & 39.349 \\
			\hline
			CNMF-SPA-EM [Proposed] & \textbf{99.47$\pm$0.68} & 0.210 \\
			\hline
			CTD [Kargas \textit{et al.}] & 96.09$\pm$0.46 & 24.116 \\
			\hline
			SVM   & 97.60$\pm$0.28 & 34.443 \\
			\hline
			Logistic Regression & 96.75$\pm$0.62 & 2.875 \\
			\hline
			%	Neural Net  & 99.24$\pm$0.45 & 0.979 \\
			%	\hline
			SVM-RBF & 97.16$\pm$0.41 & 1.523 \\
			\hline
			Naive Bayes & 94.60$\pm$0.56 & 0.063 \\
			\hline
		\end{tabular}%
	}
	\label{tab:mushroom}%
	%\\{\tiny the results for MLE-AD-EM is not reported due to very large runtime}
		\vspace{-.25cm}
\end{table}%

\subsection{Real-data Experiments: Recommender Systems}	
We test the approaches using the MovieLens 20M dataset \cite{harper2015movielens}. Again, following the evaluation strategy in \cite{kargas2017tensors}, we first round the ratings to the closest integers so that every movie's rating resides in $\{1,2,\ldots,5\}$. We choose different movie genres, namely, action, animation and romance, and select a subset of movies from each genre. Each subset contains 30 movies. Hence, for every subset, $N=30$, $Z_i$ represents the rating of movie $i$, and $Z_i$'s alphabet is $\{1,2,\ldots,5\}$. We predict the rating for a movie, say, movie $N$, by user $k$ via computing $\mathbb{E}[i_N|r_k(1),\ldots,r_k(N-1)]$, where $r_k(i)$ denotes the rating of movie $i$ by user $k$ (i.e., using the MMSE estimator). This can be done via estimating ${\sf Pr}(i_1,\ldots,i_N)$.
Note that we leave out the MLE-AD-EM algorithm for the MovieLens datasets due to its scalability challenges.

We create the validation and testing sets by randomly hiding  $20\%$ and $30\%$ of the dataset for each trial. The remaining $50\%$ is used for training (i.e., learning the joint PMF in our approach). %The rank $F$ for all the methods and the number of iterations needed for CNMF-SPA-EM are chosen using the validation set.
The results are taken from 20 random trials as before. In addition to the root MSE (RMSE) performance, we also report the \textit{mean absolute error} (MAE) of the predicted ratings, which is more meaningful in the presence of outlying trials; see the definitions of RMSE and MAE in \cite{kargas2017tensors}.

Table~\ref{tab:action}-\ref{tab:romance} present the results for three different subsets, respectively. One can see that the proposed methods are promising---their predictions are either comparable or better than the BMF approach. Note that BMF is specialized for recommender systems, while the proposed approaches are for generic joint PMF recovery. 
The fact that our methods perform better suggests that the underlying joint PMF is well captured by the proposed CNMF approach.
%even with the sparse nature of the user-movie datasets. 

%We would like to remark that since the ground-truth joint PMFs are unknown for real data, we use prediction/classification accuracy to indirectly evaluate our methods. Notably, our generic PMF learning methods output better or comparable classification/prediction results relative to classification/specialized recommendation approaches. This suggests that the proposed framework is effective in capturing the essence of the underlying probabilistic models.

\begin{table}[t!]
	\centering
	\caption{Movie Recommendation Results on the Action set.  }
	\label{tab:action}%
	\resizebox{.6\linewidth}{!}{\Huge
		\begin{tabular}{|c|c|c|c|}
			\hline
			\textbf{Algorithm} & \textbf{RMSE} & \textbf{MAE} & \textbf{Time (s)} \\
			\hline
			CNMF-SPA [Proposed] & 0.8497$\pm$0.0114 & 0.6663$\pm$0.0059 & 0.031 \\
			\hline
			CNMF-OPT [Proposed] & 0.8167$\pm$0.0035 & 0.6321$\pm$0.0040 & 70.018 \\
			\hline
			CNMF-SPA-EM [Proposed] & \textbf{0.7840$\pm$0.0025} & \textbf{0.5991$\pm$0.0031} & 2.424 \\
			\hline
			CTD [Kargas \textit{et al.}] & 0.8770$\pm$0.0088 & 0.6649$\pm$0.0076 & 52.253 \\
			\hline
			BMF   & 0.8011$\pm$0.0012 & 0.6260$\pm$0.0013 & 46.637 \\
			\hline
			Global Average  & 0.9468$\pm$0.0018 & 0.6956$\pm$0.0017 & -- \\
			\hline
			User Average & 0.8950$\pm$0.0010 & 0.6825$\pm$0.0010 & -- \\
			\hline
			Movie Average & 0.8847$\pm$0.0018 & 0.6982$\pm$0.0012 & -- \\
			\hline
		\end{tabular}%
		
	}
	\vspace{-.25cm}	
\end{table}%
\begin{table}[t!]
	\centering
	\caption{Movie Recommendation Results on the Animation set.}
	\resizebox{.6\linewidth}{!}{\Huge
		\begin{tabular}{|c|c|c|c|}
			\hline
			\textbf{Algorithm} & \textbf{RMSE} & \textbf{MAE} & \textbf{Time (s)} \\
			\hline
			CNMF-SPA [Proposed] & 0.8705$\pm$0.0095 & 0.6798$\pm$0.0060 & 0.028 \\
			\hline
			CNMF-OPT [Proposed] & \textbf{0.8124$\pm$0.0031} & \textbf{0.6241$\pm$0.0041} & 61.018 \\
			\hline
			CNMF-SPA-EM [Proposed] & 0.8170$\pm$0.0075 & 0.6317$\pm$0.0086 & 2.424 \\
			\hline
			CTD [Kargas \textit{et al.}] & 0.8300$\pm$0.0053 & 0.6335$\pm$0.0029 & 48.253 \\
			\hline
			BMF   & 0.8408$\pm$0.0023 & 0.6553$\pm$0.0015 & 46.637 \\
			\hline
			Global Average  & 0.9371$\pm$0.0021 & 0.7042$\pm$0.0014 & -- \\
			\hline
			User Average & 0.8850$\pm$0.0009 & 0.6632$\pm$0.0011 & -- \\
			\hline
			Movie Average & 0.9027$\pm$0.0019 & 0.6900$\pm$0.0013 & -- \\
			\hline
		\end{tabular}%
	}
	\label{tab:anim}%
	\vspace{-.25cm}
\end{table}%

\begin{table}[t!]
	\centering
	\caption{Movie Recommendation Results on the Romance set.}
	\resizebox{.6\linewidth}{!}{\Huge
		\begin{tabular}{|c|c|c|c|}
			\hline
			\textbf{Algorithm} & \textbf{RMSE} & \textbf{MAE} & \textbf{Time (s)} \\
			\hline
			CNMF-SPA [Proposed] & 0.9280$\pm$0.0066 & 0.7376$\pm$0.0076 & 0.032 \\
			\hline
			CNMF-OPT [Proposed] & 0.9076$\pm$0.0014 & 0.7123$\pm$0.0029 & 60.762 \\
			\hline
			CNMF-SPA-EM [Proposed] & \textbf{0.9057$\pm$0.0052} & \textbf{0.7106$\pm$0.0049} & 1.881 \\
			\hline
			CTD [Kargas \textit{et al.}] & 0.9498$\pm$0.0085 & 0.7416$\pm$0.0054 & 47.010 \\
			\hline
			BMF   & 0.9337$\pm$0.0007 & 0.7463$\pm$0.0009 & 31.823 \\
			\hline
			Global Average  & 1.0019$\pm$0.0007 & 0.8078$\pm$0.0008 & -- \\
			\hline
			User Average & 1.0195$\pm$0.0007 & 0.7862$\pm$0.0008 &  --\\
			\hline
			Movie Average & 0.9482$\pm$0.0007 & 0.7599$\pm$0.0007 &  --\\
			\hline
		\end{tabular}%
	}
	\label{tab:romance}%
	\vspace{-.25cm}
\end{table}%

\section{Conclusion}

We proposed a new framework for recovering joint PMF of a finite number of discrete RVs from marginal distributions. Unlike a recent approach that relies on three-dimensional marginals, our approach only uses two-dimensional marginals, which naturally has reduced sample complexity and a lighter computational burden.
We showed that under certain conditions, the recoverability of the joint PMF can be guaranteed---with only noisy pairwise marginals accessible. We proposed a coupled NMF formulation as the optimization surrogate for this task, and proposed to employ a Gram--Schmidt-like scalable algorithm as its initialization. 
We showed that the initialization method is effective even under the finite-sample case (whereas the existing tensor-based method does not offer sample complexity characterization). We also showed that the Gram--Schmidt-like algorithm can provably enhance the performance of an EM algorithm that is proposed from an ML estimation perspective for the joint PMF recovery problem.
We tested the proposed approach on classification tasks and recommender systems, and the results corroborate our analyses.

\color{black}

\bibliographystyle{IEEEtran}

%\bibliography{refs}

%\bibliography{refs}

\begin{thebibliography}{10}
\providecommand{\url}[1]{#1}
\csname url@samestyle\endcsname
\providecommand{\newblock}{\relax}
\providecommand{\bibinfo}[2]{#2}
\providecommand{\BIBentrySTDinterwordspacing}{\spaceskip=0pt\relax}
\providecommand{\BIBentryALTinterwordstretchfactor}{4}
\providecommand{\BIBentryALTinterwordspacing}{\spaceskip=\fontdimen2\font plus
\BIBentryALTinterwordstretchfactor\fontdimen3\font minus
  \fontdimen4\font\relax}
\providecommand{\BIBforeignlanguage}[2]{{%
\expandafter\ifx\csname l@#1\endcsname\relax
\typeout{** WARNING: IEEEtran.bst: No hyphenation pattern has been}%
\typeout{** loaded for the language `#1'. Using the pattern for}%
\typeout{** the default language instead.}%
\else
\language=\csname l@#1\endcsname
\fi
#2}}
\providecommand{\BIBdecl}{\relax}
\BIBdecl

\bibitem{trees2013detection}
H.~L. {Van Trees}, K.~L. {Bell}, and Z.~{Tian}, \emph{Detection Estimation and
  Modulation Theory, Part I: Detection, Estimation, and Filtering Theory, 2nd
  Edition}.\hskip 1em plus 0.5em minus 0.4em\relax John Wiley and Sons, New
  York, NY, 2013.

\bibitem{ephraim1992bayesian}
Y.~{Ephraim}, ``A {B}ayesian estimation approach for speech enhancement using
  hidden markov models,'' \emph{IEEE Trans. Signal Process.}, vol.~40, no.~4,
  pp. 725--735, 1992.

\bibitem{colavolpe2005map}
G.~{Colavolpe} and A.~{Barbieri}, ``On {MAP} symbol detection for {ISI}
  channels using the {U}ngerboeck observation model,'' \emph{IEEE Commun.
  Lett.}, vol.~9, no.~8, pp. 720--722, 2005.

\bibitem{rangan2012asymptotic}
S.~{Rangan}, A.~K. {Fletcher}, and V.~K. {Goyal}, ``Asymptotic analysis of
  {MAP} estimation via the replica method and applications to compressed
  sensing,'' \emph{IEEE Trans. Inf. Theory}, vol.~58, no.~3, pp. 1902--1923,
  2012.

\bibitem{farag2005unified}
A.~A. {Farag}, R.~M. {Mohamed}, and A.~{El-Baz}, ``A unified framework for
  {MAP} estimation in remote sensing image segmentation,'' \emph{IEEE Trans.
  Geosci. Remote Sens.}, vol.~43, no.~7, pp. 1617--1634, 2005.

\bibitem{cover1999elements}
T.~M. Cover, \emph{Elements of information theory}.\hskip 1em plus 0.5em minus
  0.4em\relax Hoboken, NJ, USA: John Wiley \& Sons, 1999.

\bibitem{taleb1999source}
A.~Taleb and C.~Jutten, ``Source separation in post-nonlinear mixtures,''
  \emph{IEEE Trans. Signal Process.}, vol.~47, no.~10, pp. 2807--2820, 1999.

\bibitem{thevenaz2000optimization}
P.~{Thevenaz} and M.~{Unser}, ``Optimization of mutual information for
  multiresolution image registration,'' \emph{IEEE Trans. Image Process.},
  vol.~9, no.~12, pp. 2083--2099, 2000.

\bibitem{yang2007mimo}
Y.~Yang and R.~S. Blum, ``Mimo radar waveform design based on mutual
  information and minimum mean-square error estimation,'' \emph{IEEE Trans.
  Aerosp. Electron. Syst. (}, vol.~43, no.~1, pp. 330--343, 2007.

\bibitem{murphy2012machine}
K.~P. Murphy, \emph{Machine learning: a probabilistic perspective}.\hskip 1em
  plus 0.5em minus 0.4em\relax Cambridge, MA, USA: MIT press, 2012.

\bibitem{yeredor2019maximum}
A.~{Yeredor} and M.~{Haardt}, ``Maximum likelihood estimation of a low-rank
  probability mass tensor from partial observations,'' \emph{IEEE Signal
  Process. Lett.}, vol.~26, no.~10, pp. 1551--1555, Oct 2019.

\bibitem{kargas2017tensors}
N.~Kargas, N.~D. Sidiropoulos, and X.~Fu, ``Tensors, learning, and
  '{K}olmogorov extension'for finite-alphabet random vectors,'' \emph{IEEE
  Trans. Signal Process.}, vol.~66, no.~18, pp. 4854--4868, 2018.

\bibitem{dabeer2013joint}
O.~{Dabeer}, ``Joint probability mass function estimation from asynchronous
  samples,'' \emph{IEEE Trans. Signal Process.}, vol.~61, no.~2, pp. 355--364,
  2013.

\bibitem{hunsop2008newapproach}
{Hunsop Hong} and D.~{Schonfeld}, ``A new approach to constrained
  expectation-maximization for density estimation,'' in \emph{IEEE
  International Conference on Acoustics, Speech and Signal Processing}, 2008,
  pp. 3689--3692.

\bibitem{jenq1994nonparametric}
{Jenq-Neng Hwang}, {Shyh-Rong Lay}, and A.~{Lippman}, ``Nonparametric
  multivariate density estimation: a comparative study,'' \emph{IEEE Trans.
  Signal Process.}, vol.~42, no.~10, pp. 2795--2810, 1994.

\bibitem{wainwright2019high}
M.~J. Wainwright, \emph{High-dimensional statistics: A non-asymptotic
  viewpoint}.\hskip 1em plus 0.5em minus 0.4em\relax Cambridge, UK: Cambridge
  University Press, 2019, vol.~48.

\bibitem{kay1993fundamentals}
S.~M. Kay, \emph{Fundamentals of statistical signal processing}.\hskip 1em plus
  0.5em minus 0.4em\relax Upper Saddle River, NJ, USA: Prentice Hall PTR, 1993.

\bibitem{bishop2006pattern}
C.~M. Bishop, \emph{Pattern Recognition and Machine Learning (Information
  Science and Statistics)}.\hskip 1em plus 0.5em minus 0.4em\relax Berlin,
  Heidelberg: Springer-Verlag, 2006.

\bibitem{jordan1998learning}
M.~I. Jordan, \emph{Learning in graphical models}.\hskip 1em plus 0.5em minus
  0.4em\relax Cambridge, MA, USA: MIT Press, 1998, vol.~89.

\bibitem{tipping2001sparse}
M.~E. Tipping, ``Sparse {B}ayesian learning and the relevance vector machine,''
  \emph{J. Mach. Learn. Res.}, vol.~1, pp. 211--244, Sep. 2001.

\bibitem{blei2003latent}
D.~M. Blei, A.~Y. Ng, and M.~I. Jordan, ``Latent dirichlet allocation,''
  \emph{J. Mach. Learn. Res.}, vol.~3, pp. 993--1022, Mar. 2003.

\bibitem{hillar2013most}
C.~J. Hillar and L.-H. Lim, ``Most tensor problems are {NP}-hard,''
  \emph{Journal of the ACM (JACM)}, vol.~60, no.~6, p.~45, 2013.

\bibitem{ibrahim2020recover}
S.~Ibrahim and X.~Fu, ``Recovering joint {PMF} from pairwise marginals,'' in
  \emph{54th Asilomar Conference on Signals, Systems and Computers}, Oct. 2020,
  pp. 356--360.

\bibitem{harshman1970foundations}
R.~Harshman, ``Foundations of the parafac procedure: Models and conditions for
  an ``explanatory'' multi-modal factor analysis,'' \emph{UCLA Working Papers
  in Phonetics}, vol.~16, 1970.

\bibitem{sidiropoulos2017tensor}
N.~D. Sidiropoulos, L.~De~Lathauwer, X.~Fu, K.~Huang, E.~E. Papalexakis, and
  C.~Faloutsos, ``Tensor decomposition for signal processing and machine
  learning,'' \emph{IEEE Trans. Signal Process.}, vol.~65, no.~13, pp.
  3551--3582, 2017.

\bibitem{han2015minimax}
Y.~Han, J.~Jiao, and T.~Weissman, ``Minimax estimation of discrete
  distributions under l1 loss,'' \emph{IEEE Trans. Info. Theory}, vol.~61,
  no.~11, pp. 6343--6354, 2015.

\bibitem{fu2018nonnegative}
X.~Fu, K.~Huang, N.~D. Sidiropoulos, and W.-K. Ma, ``Nonnegative matrix
  factorization for signal and data analytics: Identifiability, algorithms, and
  applications,'' \emph{IEEE Signal Process. Mag.}, vol.~36, no.~2, pp. 59--80,
  March 2019.

\bibitem{gillis2014and}
N.~Gillis, ``The why and how of nonnegative matrix factorization,''
  \emph{Regularization, Optimization, Kernels, and Support Vector Machines},
  vol.~12, pp. 257--291, 2014.

\bibitem{donoho2003does}
D.~Donoho and V.~Stodden, ``When does non-negative matrix factorization give a
  correct decomposition into parts?'' in \emph{Advances in Neural Information
  Processing Systems}, 2003, pp. 1141--1148.

\bibitem{Gillis2012}
N.~Gillis and S.~Vavasis, ``Fast and robust recursive algorithms for separable
  nonnegative matrix factorization,'' \emph{IEEE Trans. Pattern Anal. Mach.
  Intell.}, vol.~36, no.~4, pp. 698--714, April 2014.

\bibitem{fu2014self}
X.~Fu, W.-K. Ma, T.-H. Chan, and J.~M. Bioucas-Dias, ``Self-dictionary sparse
  regression for hyperspectral unmixing: Greedy pursuit and pure pixel search
  are related,'' \emph{IEEE J. Sel. Topics Signal Process.}, vol.~9, no.~6, pp.
  1128--1141, 2015.

\bibitem{arora2012practical}
S.~Arora, R.~Ge, Y.~Halpern, D.~Mimno, A.~Moitra, D.~Sontag, Y.~Wu, and M.~Zhu,
  ``A practical algorithm for topic modeling with provable guarantees,'' in
  \emph{Proceedings of International Conference on Machine Learning}, 2013.

\bibitem{VMAX}
T.-H. Chan, W.-K. Ma, A.~Ambikapathi, and C.-Y. Chi, ``A simplex volume
  maximization framework for hyperspectral endmember extraction,'' \emph{IEEE
  Trans. Geosci. Remote Sens.}, vol.~49, no.~11, pp. 4177--4193, Nov. 2011.

\bibitem{MC01}
U.~Ara\'ujo, B.~Saldanha, R.~Galv\~ao, T.~Yoneyama, H.~Chame, and V.~Visani,
  ``The successive projections algorithm for variable selection in
  spectroscopic multicomponent analysis,'' \emph{Chemometrics and Intelligent
  Laboratory Systems}, vol.~57, no.~2, pp. 65--73, 2001.

\bibitem{huang2014non}
K.~Huang, N.~Sidiropoulos, and A.~Swami, ``Non-negative matrix factorization
  revisited: Uniqueness and algorithm for symmetric decomposition,'' \emph{IEEE
  Trans. Signal Process.}, vol.~62, no.~1, pp. 211--224, 2014.

\bibitem{fu2015blind}
X.~Fu, W.-K. Ma, K.~Huang, and N.~D. Sidiropoulos, ``Blind separation of
  quasi-stationary sources: Exploiting convex geometry in covariance domain,''
  \emph{IEEE Trans. Signal Process.}, vol.~63, no.~9, pp. 2306--2320, May 2015.

\bibitem{fu2016robust}
X.~{Fu}, K.~{Huang}, B.~{Yang}, W.~{Ma}, and N.~D. {Sidiropoulos}, ``Robust
  volume minimization-based matrix factorization for remote sensing and
  document clustering,'' \emph{IEEE Trans. Signal Process.}, vol.~64, no.~23,
  pp. 6254--6268, Dec 2016.

\bibitem{fu2018identifiability}
X.~Fu, K.~Huang, and N.~D. Sidiropoulos, ``On identifiability of nonnegative
  matrix factorization,'' \emph{IEEE Signal Process. Lett.}, vol.~25, no.~3,
  pp. 328--332, 2018.

\bibitem{ibrahim2019crowdsourcing}
S.~Ibrahim, X.~Fu, N.~Kargas, and K.~Huang, ``Crowdsourcing via pairwise
  co-occurrences: Identifiability and algorithms,'' in \emph{Advances in Neural
  Information Processing Systems}, 2019, pp. 7845--7855.

\bibitem{kargas2019learning}
N.~Kargas and N.~D. Sidiropoulos, ``Learning mixtures of smooth product
  distributions: Identifiability and algorithm,'' in \emph{The 22nd
  International Conference on Artificial Intelligence and Statistics}, 2019,
  pp. 388--396.

\bibitem{yang2019learning}
B.~{Yang}, X.~{Fu}, N.~D. {Sidiropoulos}, and K.~{Huang}, ``Learning nonlinear
  mixtures: Identifiability and algorithm,'' \emph{IEEE Trans. Signal
  Process.}, vol.~68, pp. 2857--2869, 2020.

\bibitem{wedin1973perturbation}
P.-{\AA}. Wedin, ``Perturbation theory for pseudo-inverses,'' \emph{BIT
  Numerical Mathematics}, vol.~13, no.~2, pp. 217--232, Jun 1973.

\bibitem{zhang2014spectral}
Y.~Zhang, X.~Chen, D.~Zhou, and M.~I. Jordan, ``Spectral methods meet {EM}: A
  provably optimal algorithm for crowdsourcing,'' \emph{J. Mach. Learn. Res.},
  vol.~17, no. 102, pp. 1--44, 2016.

\bibitem{koren2009matrix}
Y.~Koren, R.~Bell, and C.~Volinsky, ``Matrix factorization techniques for
  recommender systems,'' \emph{Computer}, vol.~42, no.~8, pp. 30--37, Aug.
  2009.

\bibitem{harper2015movielens}
F.~M. Harper and J.~A. Konstan, ``The movielens datasets: History and
  context,'' \emph{ACM Trans. Interact. Intell. Syst.}, vol.~5, no.~4, pp.
  19:1--19:19, Dec. 2015.

\end{thebibliography}

%\onecolumn

\appendix

\section{Proof of Theorem~\ref{prop:prop2}}\label{app:prop1}
The first part of the proof is reminiscent of the identifiability proof of a similar problem in \cite{ibrahim2019crowdsourcing}.
The difference lies in how to handle ``fat'' $\A_n$'s. In \cite{ibrahim2019crowdsourcing}, the $\A_n$'s are all square nonsingular matrices.
In our case, $\A_n$'s can be `fat' matrices. This creates some more challenges.

Step 1):\
Under our assumption, there exists a splitting denoted by ${\cal S}_1 = \{\ell_1,\ldots,\ell_{M}\},~{\cal S}_2 =  \{\ell_{M+1},\ldots,\ell_{Q}\}$ $
{\cal S}_1\cup {\cal S}_1\subseteq [N],~{\cal S}_1\cap{\cal S}_2 = \emptyset$
such that
\begin{equation}\label{eq:constructedX_proof}
\begin{aligned}
\widetilde{\bm X}=& \begin{bmatrix} \bm X_{\ell_1,\ell_{M+1}}& \ldots & {\bm X}_{\ell_1,\ell_Q} \\ \vdots & \vdots &\vdots \\ {\bm X}_{\ell_M,\ell_{M+1}} & \ldots & {\bm X}_{\ell_M,\ell_Q} \end{bmatrix} = \bm W\bm H^{\top}, %\\
%= &\underbrace{\begin{bmatrix}
%	\bm A_{\ell_1}\\ \vdots\\ \bm %A_{\ell_M}
%	\end{bmatrix}}_{\bm %W}\underbrace{{\bm D}({\bm \lambda %})[{\bm %A}_{\ell_{M+1}}^\T,\ldots,{\bm %A}_{\ell_{Q}}^\T]}_{{\bm H}^\T}.
\end{aligned}
\end{equation}
where $\W=[\A_{\ell_1}^\T,\ldots,\A_{\ell_M}^\T]^\T $ and {$\H=[\A_{\ell_{M+1}}^\T,\ldots,\A_{\ell_{Q}}^\T]^\T\bm D(\bm \lambda)$}.
Note that $Q<N$ is allowed. Since we have assumed that $\W$ and $\H$ are sufficiently scattered (since  $\|\bm \lambda\|_0=F$ by the assumption that ${\sf Pr}(f)\neq 0$ for all $f$). By Theorem~4 in \cite{huang2014non}, one can see that by solving \eqref{eq:coupled_mat}, we always have
$\widehat{\W}=\W\bm D\bm \Pi$ and $\widehat{\bm H}=\H\bm \Pi\bm D^{-1}$, where $\widehat{\W}=[\widehat{\A}_{\ell_1}^\T,\ldots,\widehat{\A}_{\ell_M}^\T]^\T$ and $\widehat{\H}=[\widehat{\A}_{\ell_{M+1}}^\T,\ldots,\widehat{\A}_{\ell_Q}^\T]^\T$.
Since the column norms of $\A_n$'s are known, by column normalization of each $\widehat{\bm A}_n$ with respect to $\ell_1$-norm, we have
\begin{equation}\label{eq:inS}
\widehat{\bm A}_n = \bm A_n\bm \Pi,~\forall n\in{\cal S}_1\cup{\cal S}_2  .    
\end{equation}

Step 2):\ Now we show that $\A_j$ for $j\notin {\cal S}_1\cup{\cal S}_2 $ can also be identified up to the same permutation ambiguity; i.e., we hope to show that any $\widehat{\A}_j$ that is a solution of \eqref{eq:coupled_mat} satisfies $\widehat{\A}_j =\A_j\bm \Pi$, with the same $\bm \Pi$ as in \eqref{eq:inS}.

Let us denote $\widehat{\A}_n$ and $\widehat{\bm \lambda}$ as any optimal solution of Problem~\eqref{eq:coupled_mat}.
Since there exists a construction \eqref{eq:constructedX_proof}, and $\widehat{\A}_{\ell_q}=\A_{\ell_q}\bm \Pi$, $q \in \{1,\hdots,Q\}$ can be identified by our construction, it is easy to see that $\widehat{\bm \lambda}=\bm \Pi^\T \bm \lambda.$
Indeed, from \eqref{eq:constructedX_proof}, one can see that any optimal solution of Problem~\eqref{eq:coupled_mat} satisfies
\begin{equation}\label{eq:constructedX_proof2}
\begin{aligned}
\widetilde{\bm X}
= &\underbrace{\begin{bmatrix}
	\widehat{\bm A}_{\ell_1}\\ \vdots\\ \widehat{\bm A}_{\ell_M}
	\end{bmatrix}}_{\bm Z_1}{\bm D}(\widehat{\bm \lambda })\underbrace{[\widehat{\bm A}_{\ell_{M+1}}^\T,\ldots,\widehat{\bm A}_{\ell_{Q}}^\T]}_{{\bm Z}_2^\T}.
\end{aligned}
\end{equation}
Again, since $\Z_1$ and $\Z_2$ are sufficient scattered, we have ${\rm rank}(\Z_1)={\rm rank}(\Z_2)=F$. Consequently, we have
$     \widehat{\bm \lambda} = \left( \bm Z_2 \odot\bm Z_1 \right)^{\dagger} {\rm vec}(\widetilde{\bm X}) = \bm \Pi^\T\bm \lambda,$
due to \eqref{eq:constructedX_proof} and the fact that 
$  {\rm rank}( \bm Z_2 \odot\bm Z_1)=F, $
which is a result of ${\rm rank}(\Z_1)={\rm rank}(\Z_2)=F$ \cite{sidiropoulos2017tensor}.
In the above, `$\odot$' denotes the Khatri-Rao product.

Note that by condition iii), we have $r_1,\ldots,r_T \in {\cal S}_1\cup {\cal S}_2$ such that
$ \X_{r_tj} =\widehat{\A}_{r_t} \bm D(\widehat{\bm \lambda})\widehat{\bm A}_{j}^\T,$
for $t=1,\ldots,T$.
These equalities can be re-expressed as follows
\begin{equation}
\begin{bmatrix}
\X_{r_1j}\\
\vdots\\
\X_{r_Tj}
\end{bmatrix}
=
\begin{bmatrix}
{\A}_{r_1}\\
\vdots\\
{\A}_{r_T}
\end{bmatrix}{\bm D}({\bm \lambda}){\bm A}_j^\T
=
\begin{bmatrix}
\widehat{\A}_{r_1}\\
\vdots\\
\widehat{\A}_{r_T}
\end{bmatrix}{\bm D}(\widehat{\bm \lambda})\widehat{\bm A}_j^\T.
\end{equation}
Note that 
$  {\bm D}(\widehat{\bm \lambda}) = \bm \Pi^\T \bm D(\bm \lambda) \bm \Pi^\T, $
since $\widehat{\bm \lambda}=\bm \Pi^\T\bm \lambda$.
By the assumption that any $T$-concatenation of $\A_n$'s have full column rank and $\widehat{\bm A}_{r_t}=\A_{r_t}\bm \Pi$, one can see that
$\widehat{\bm A}_j = \A_j\bm \Pi,~j\notin{\cal S}_1\cup {\cal S}_2.$
Since we have assumed that for every $j\notin{\cal S}_1\cup {\cal S}_2$, there exists a set of $r_1,\ldots,r_T$ such that condition iii) is satisfied.

Once $\bm A_n$'s and $\bm \lambda$ are estimated, the joint probability ${\sf Pr}(i_1,i_2,\ldots,i_N), i_n \in \{1,\dots, I_n\}$ can be estimated by
\begin{equation}\label{eq:pmf_latent_var1}
\begin{aligned}
\widehat{\sf Pr}(i_1,i_2,\ldots,i_N) &= \sum_{f=1}^F \widehat{\bm \lambda}(f) \prod_{n=1}^N \widehat{\bm  A}_n(i_n , f)\\
&=\sum_{f=1}^F {\bm \lambda}(f) \prod_{n=1}^N{\bm  A}_n(i_n , f).
\end{aligned}
\end{equation}  
The last equality holds because the unified permutation ambiguity across $\bm A_n$'s and $\bm \lambda$ does not affect the reconstruction of $\widehat{\sf Pr}(i_1,i_2,\ldots,i_N)$.

\section{Proof of Theorem~\ref{thm:spabound}}\label{app:spa}

Consider a matrix factorization model as below:
\begin{align} \label{eq:noisynmf}
\widetilde{\X}=\W\H^\T,
\end{align}
where ${\bm W}\in\mathbb{R}^{L\times F}$, $\bm H\in\mathbb{R}^{K\times F}$, $\bm W\geq \bm 0$ and $\bm H\geq \bm 0$.

The SPA algorithm for factoring $\widetilde{\X}$ into $\W$ and $\H$ consists of two key steps:
\begin{mdframed}
\begin{enumerate}
    \item Normalize the columns of $\widetilde{\X}$ w.r.t. their $\ell_1$ norms.
    \item Apply an $F$ step Gram-Schmidt-like procedure to pick up $\bm \varLambda=\{l_1,\ldots,l_F\}$.
\end{enumerate}
\end{mdframed}
Note that $\widetilde{\X}(:,\bm \varLambda)= \bm W\bm \Sigma$; i.e., $\W$ can be identified up to a column scaling ambiguity through this simple procedure.
Gillis and Vavasis \cite{Gillis2012} have shown that under the model in \eqref{eq:noisynmf}, SPA is provably robust to noise in estimating the factor matrix $\bm W$ (see Lemma~\ref{lem:gillis} in Appendix~\ref{app:lemmata} from \cite{Gillis2012}). To proceed, first characterize the ``noise'' in our virtual NMF model.

Consider the pairwise marginals $\bm X_{jk}$'s which are used to construct the matrix $\widetilde{\X}$ in \eqref{eq:constructedX}.
$\bm X_{jk}$'s are estimated by sample averaging of a finite number of realizations and thus the estimated $\bm X_{jk}$ (denoted as $\widehat{\bm X}_{jk }$) is always noisy; i.e., we have
\begin{align}
\widehat{\bm X}_{jk } = \bm X_{jk} +{\bm N}_{jk},
\end{align}
where the noise matrix ${\bm N}_{jk} \in \mathbb{R}^{I\times I}$, assuming $I_n=I$ for all $n \in \{1,\dots,N\}$.
We have the following proposition:
\begin{proposition} \label{prop:sample}
	Let $p \in (0,1]$ be the probability that an RV is observed. Let $S$ be the number of available realizations of $N$ RVs. Assume that $ p \ge (\frac{8}{S}\log(2/\delta))^{1/2}$. Then, with probability at least $1-\delta$, 
	$
	\| \bm X_{jk}-\widehat{ \bm X}_{jk}\|_{\rm F}=\|{\bm N}_{jk}\|_{\rm F} \leq \phi,$
	holds for any distinct $j,k$ where $\phi = \frac{\sqrt{2}(1+\sqrt{{\rm log}(2/\delta)})}{(p\sqrt{S})}$.
\end{proposition} 
The proof of Proposition \ref{prop:sample} is given in Sec.~\ref{app:prop:sample}. 
%To proceed, assume that we have enough number of co-realizations such that $\|\bm N_{ij}\|_F \le \phi$, then the below holds:

By the definition of the Frobenius norm, we have
$
\sum_{c=1}^{I}\|{\bm N}_{ij }(:,c)\|^2_2 =\|{\bm N}_{ij }\|^2_{\rm F} \le \phi^2.
$
Applying norm equivalence $\frac{\|{\bm N}_{ij }(:,c)\|_1}{\sqrt{I}} \le \|{\bm N}_{ij}(:,c)\|_2$, we get
\begin{align}
\|{\bm N}_{ij }(:,c)\|^2_1 \le I\phi^2 \implies \|{\bm N}_{ij }(:,c)\|_1 \le  \sqrt{I}\phi. \label{eq:Nnormij}
\end{align}

By using the estimates $\widehat{\bm X}_{jk }$, the model given by  \eqref{eq:constructedX} can be represented as
\begin{equation}\label{eq:constructedX1}
\begin{aligned}
\widehat{\bm X}= \widetilde{\bm X}+ \widetilde{\bm N},
\end{aligned}
\end{equation}
where the $(i,j)$th block of $\widetilde{\bm N}$ is $\bm N_{ij}$.
%Let us also denote $\bm H' = \begin{bmatrix}\bm A_{\ell_{M+1}}^{\top}&\hdots&\bm A_{\ell_{N}}^{\top}\end{bmatrix}^{\top}$.
Note that $\widehat{\bm X},\widetilde{\bm X}$ and $\widetilde{\bm N}$ all have the same size of ${L \times K}$. 
Assuming $I_n=I$ for all $n \in \{1,\dots,N\}$, we have $L=MI$ and $K=(N-M)I$. 
Also note that $\bm W$ has a size of ${L\times F}$ and $\bm H$ has a size of ${K \times F}$.

%In order to estimate $\bm W$ and $\bm H$, we will apply SPA to the matrix $\widehat{\bm X} = \bm W \bm H^{\top}+\widetilde{\bm N}$. For this, the first step is to normalize the columns of $\widehat{\bm X}$ with respect to $\ell_1$-norm. The estimation accuracy of $\bm W$ and $\bm H$ can then be characterized by using the following lemma:

{Since any column of $\widetilde{\bm N}$ is formed} from the columns of $M$ number of $\bm N_{ij}$'s,
we have 
\begin{align}
\|\widetilde{\bm N}(:,q)\|_1 
& \le M\sqrt{I}\phi, \label{eq:tildeN}
\end{align}
by the triangle inequality and \eqref{eq:Nnormij}.

As mentioned, before performing SPA, the columns of $\widehat{\bm X}$ are normalized with respect to the $\ell_1$-norm. 
In the noiseless case, 
let us denote $\overline{\X}(:,q) = \frac{\widetilde{\bm X}(:,q)}{\|  \widetilde{\bm X}(:,q)  \|_1 } $. Then, we have
\[  \overline{\X}(:,q) = \sum_{f=1}^F \frac{ \W(:,f) }{\| \W(:,f) \|_1} \frac{ \| \W(:,f) \|_1 \H(q,f) }{\|\sum_{f=1}^F \W(:,f)\H(q,f)\|_1 },      \]
or equivalently $\overline{\X}=\overline{\W}\overline{\H}^\T$,
where $\overline{\W}(:,f)=\frac{ \W(:,f) }{\| \W(:,f) \|_1}$ and $\overline{\H}(q,r)=\frac{ \| \W(:,f) \|_1 \H(q,f) }{\|\sum_{f=1}^F \W(:,f)\H(q,f)\|_1 }$.
One can verify that $\overline{\H}\bm 1 =\bm 1^\T$, which is critical for applying SPA.
When the data is noisy, i.e., $\widehat{\X}=\widetilde{\X} + \widetilde{\bm N}$, we hope to show that the normalized data can be represented as follows:
\begin{align} \label{eq:datanorm}
\overline{\bm X} = \overline{\bm W}\overline{\bm H}^{\top} + \overline{\bm N},
\end{align}
where $\overline{\bm X}$ is the column normalized version (with respect to the $\ell_1$ norm) of $\widehat{\bm X}$, and $\overline{\bm W}$ and $\overline{\bm H}$ are defined as above.

From the assumption $\|\widehat{\bm X}_{ij}(:,c)\|_1 \ge \eta$ for any $i \neq j$ and $c \in \{1,\dots, I\}$, we get $\|\widehat{\bm X}(:,q)\|_1 \ge M\eta$ for any $q$. 
Combining Lemma~\ref{lem:normalization} (see Appendix~\ref{app:lemmata}) and Eq.~\eqref{eq:tildeN}, we get
\begin{align}
\|\overline{\bm N}(:,q)\|_1 &\le \frac{2\sqrt{I}\phi}{\eta}. \label{eq:overlineN}
\end{align}

{ Applying norm equivalence, we further have} $\|{\overline{\bm N}}(:,q)\|_2 \le \|{\overline{\bm N}}(:,q)\|_1$ and hence we get
\begin{align} \label{eq:Noisenorm}
\|{\overline{\bm N}}(:,q)\|_2 &\le  \frac{2\sqrt{I}\phi}{\eta}.
\end{align}
\begin{Fact}\label{lem:spa_allnoise}
	Assume that $\|{\overline{\bm N}}(:,q)\|_2 \le \varphi$ for any $q$ and that $\overline{\H}$ satisfies $\varepsilon$-separability assumption in the model \eqref{eq:datanorm}. Suppose
	\begin{align*}
	\left({\sigma_{\rm max}(\overline\W)\varepsilon} + \varphi\right)& \leq 
	\sigma_{\rm min}(\overline\W) \varrho \widetilde{\kappa}^{-1},
	%\implies 
	%\phi  &\le  \frac{\sigma_{\rm min}(\bm{W})\eta}{2\sqrt{I}} \min\left(\frac{1}{2\sqrt{F-1}},\frac{1}{4}\right)\left(1+80\kappa^2(\bm W)\right)^{-1} - \frac{ \sigma_{\rm max}({\bm{W}})\varepsilon\eta}{2\sqrt{I}M}:= \zeta. \label{eq:epsilonbound_int}
	\end{align*}
	where $\widetilde{\kappa}=\big(1+80\kappa^2(\overline\W)\big)$ and $\varrho = \text{\rm min}\left(\frac{1}{2\sqrt{F-1}},\frac{1}{4}\right)$. 
	Then, SPA identifies an index set $\widehat{\bm \varLambda}=\{\widehat{l}_1,\ldots\widehat{l}_F \}$ such that
	\begin{equation} \label{eq:WX}
	\begin{aligned}	
	\max_{1 \leq f \leq F} \min_{\widehat{l}_f \in \widehat{\bm \varLambda}} \left\|\overline{\W}(:,f) - \overline{\bm{X}}(:,\widehat{l}_f)\right\|_2 \leq \left({\sigma_{\rm max}(\overline\W)\varepsilon} + \varphi\right)\widetilde{\kappa},
	\end{aligned}
	\end{equation}
\end{Fact}
{\it Proof:}
From the assumption that $\overline{\bm H}$ satisfies $\varepsilon$-separability, there exists a set of indices $\bm \varLambda=\{l_1,\dots, l_F\}$ such that
$
\overline{\bm{H}}(\bm \varLambda,:) = \bm{I}_{F} + \bm{E},
$
$\bm{E} \in \mathbb{R}^{F \times F}$ is the error matrix with 
$
\|\bm{E}(l,:)\|_{2}  \le  \varepsilon.
$
and $\bm{I}_{F}$ is the identity matrix of size $F\times F$. 
Without loss of generality, we assume $\bm \varLambda=\{1,\ldots,F\}$. Hence,
one can write the normalized model given in \eqref{eq:datanorm} as
\begin{align*}
\overline{\bm{X}} 
&=\overline{\bm{W}} \overline{\bm H}^{\top} +{\overline{\bm N}}= \overline{\bm{W}} [\bm{I}_{F} + \bm{E}^\top , (\bm{H}^{*})^\top]+{\overline{\bm N}} \\
&= \overline{\bm{W}} [\bm{I}_{F} ,(\bm{H}^{*})^\top] + [\overline{\bm{W}} \bm{E}^\top , \bm{0}]+\overline{\bm N},
\end{align*}
where the zero matrix $\bm{0}$ has the same dimension as that of $\bm{H}^{*}$.
By defining the noise matrix $\bm N \in \mathbb{R}^{L \times K}$ such that $\bm N :=[\overline{\bm{W}} \bm{E}^\top ,\bm{0}]+{\overline{\bm N}} $, we have
$
\overline{\bm{X}} =  \overline{\bm{W}} [\bm{I}_{F}  \ \ \ (\bm{H}^{*})^\top]+\bm N.
$
Then, for any $q \in \{1,\dots,K\}$, the following inequality holds:
\begin{align}
\| \bm{N}(:,q) \|_2 & \le \|\overline{\bm{W}}\|_2 \|\bm{E}(q,:)\|_2+ \|{\overline{\bm N}}(:,q)\|_2  \nonumber\\
& \le \sigma_{\rm max}(\overline{\bm{W}})\varepsilon + \varphi. \label{n1}
\end{align}
Combining \eqref{n1} and {Lemma~\ref{lem:gillis} (see Appendix \ref{app:lemmata})}, we obtain \eqref{eq:WX} if $\left({\sigma_{\rm max}(\overline\W)\varepsilon} + \varphi\right) \leq 
\sigma_{\rm min}(\overline\W)\varrho\widetilde{\kappa}^{-1}.$
\hfill$\square$

 The square of the left hand side of~\eqref{eq:WX}  can be written as
\begin{align}
&{\frac{1}{M^2}}\sum_{m=1}^M \max_{1 \leq f \leq F}\min_{\widehat{l}_f \in \widehat{\bm \varLambda}} \left\|{\bm A}_{\ell_m}(:,f) - \widehat{\bm A}_{\ell_m}(:,f)\right\|_2^2\nonumber\\
&\ge { \frac{1}{M^2} }\max_{1 \leq f \leq F}\min_{\widehat{l}_f \in \widehat{\bm \varLambda}}\left\|{\bm A}_{\ell_m}(:,f) - \widehat{\bm A}_{\ell_m}(:,f)\right\|_2^2, \label{eq:AmW}
\end{align}
for any $m \in \{1,\dots,M\}$, where the first equality is due to $\overline{\bm W} = [\bm A_{\ell_1}^{\top},\dots,\bm A_{\ell_M}^{\top}]^{\top}/M$, in which $\widehat{\bm A}_{\ell_m}$ denotes the corresponding estimate of ${\bm A}_{\ell_m}$.

Since $\|\bm W(:,f)\|_1=M$ for any $f$, we have $\overline{\bm W} = \bm W/M$. Therefore, we have $\sigma_{\rm max}(\overline{\bm W}) = \sigma_{\rm max}({\bm W})/M$, $\sigma_{\rm min}(\overline{\bm W}) 
=\sigma_{\rm min}({\bm W})/M,$
\begin{align}
\kappa(\overline{\bm W}) = \frac{\sigma_{\max}(\overline{\bm W})}{\sigma_{\min}(\overline{\bm W})} = \kappa(\bm W). \label{eq:kappa}
\end{align}

Therefore, by combining \eqref{eq:Noisenorm}, \eqref{eq:AmW}, \eqref{eq:kappa} and {Fact~\ref{lem:spa_allnoise}}, SPA estimates {$\bm A_{\ell_m}$} for any $m \in \{1,\dots,M\}$ such that
\begin{align} \label{eq:Cbound}
%\max_{1 \leq f \leq F}\frac{1}{M}\left\|{\bm A}_{\ell_m}(:,f) - \widehat{\bm A}_{\ell_m}(:,f)\right\|_2 
%&\le \left(\frac{\sigma_{\rm max}(\bm W)\varepsilon}{M} + \frac{2\sqrt{I}\phi}{\eta}\right)\left(1+80\kappa^2(\bm W)\right),
\max_{1 \leq f \leq F}\min_{\widehat{l}_f \in \widehat{\bm \varLambda}}&\left\|{\bm A}_{\ell_m}(:,f) - \widehat{\bm A}_{\ell_m}(:,f)\right\|_2 \nonumber \\
&\le \left({\sigma_{\rm max}(\bm W)\varepsilon} + \frac{2M\sqrt{I}\phi}{\eta}\right)\widetilde{\kappa},
\end{align}
if the below condition holds:
\begin{align}
\frac{ \sigma_{\rm max}({\bm{W}})}{M}\varepsilon + \frac{2\sqrt{I}\phi}{\eta}& \leq 
\sigma_{\rm min}(\bm{W})\varrho\widetilde{\kappa}^{-1}.\label{eq:epsiloncondition}
\end{align}

Letting $\varepsilon = \frac{M\varrho}{2\kappa(\bm{W})\widetilde{\kappa}}$, from \eqref{eq:epsiloncondition}, we get the condition on $\phi$ as follows:
\begin{align}
\phi \le \frac{\eta\sigma_{\rm max}({\bm{W}})\varepsilon}{4M\sqrt{I}} \label{eq:phi_spa}.
\end{align}

From Proposition \ref{prop:sample}, we have $\phi \leq \frac{\sqrt{2}(1+\sqrt{{\rm log}(2/\delta)})}{p\sqrt{S}}$ with probability greater than or equal to $1-\delta$. Combining with \eqref{eq:phi_spa}, we get the number of realizations $S$ required to get the estimation error bound \eqref{eq:Cbound} as below:
\begin{align*}
S \ge \frac{32M^2I(1+\sqrt{{\rm log}(2/\delta)})^2}{\sigma^2_{\rm max}({\bm{W}})\eta^2\varepsilon^2p^2}.
\end{align*}

Note that Fact~\ref{lem:spa_allnoise} holds if $\overline\H$ satisfies $\varepsilon$-separability condition.
By combining Assumption~\ref{ass:A} and Lemma~\ref{lem:Lm} (see Appendix~\ref{app:lemmata}) , we get that $\overline{\bm H}$ satisfies $\varepsilon$-separability assumption with probability greater than $1-\rho$, if \begin{align} \label{eq:Nbound}
(N-M)I = \Omega\left(\frac{\varepsilon^{-2(F-1)}}{F}{\rm log}\left(\frac{F}{\rho}\right)\right).
\end{align}
By substituting $\phi$ in \eqref{eq:Cbound} and using the fact that for any matrix $\bm A \in \mathbb{R}^{I\times F}$, the matrix 2-norm $\|{\bm A}\|_2 \le \sqrt{F}\underset{1 \leq f \leq F}{\max}\|{\bm A}(:,f)\|_2$, we get the result \eqref{eq:thmspa} in the theorem.
Letting $\rho = \frac{\delta}{2}$ in \eqref{eq:Nbound}, the proof is completed.

%Finally, we combine the probabilities involved in the results used in our proof. 
%For the concentration bound in Proposition \ref{prop:sample} and  $\varepsilon$-separability condition on $\overline\H$ given by Lemma~\ref{lem:Lm} to jointly occur with probability greater than $1-\delta$, one can let $\rho = \frac{\delta}{2}$ in \eqref{eq:Nbound}. 

\subsection{Proof Proposition~\ref{prop:sample}}\label{app:prop:sample}
Recall $\bm d_s\in\mathbb{R}^N$ denotes the $s$th realization of the joint PMF ${\sf Pr}(Z_1,\ldots,Z_N)$.
For simplicity, we assume that 0 does not belong to the alphabets of $Z_1,\ldots,Z_N$, and we use the notation $\bm d_s(j)=0$ to represent that `$Z_j$ is not observed in the $s$th realization'.

For $S$ realizations of the joint PMF, i.e., $\{\bm d_s\}_{s=1}^S$, the sample averaging expressions for estimating ${\X}_{jk}$ is defined as follows:
\begin{align*}
\widehat{\bm X}_{jk}(i_j,i_k) &= \frac{1}{|\mathcal{S}_{jk}|}\sum_{s \in \mathcal{S}_{jk}} \mathbb{I}\left[\bm d_s(j)=z_j^{(i_j)},\bm d_s(k)=z_k^{(i_k)}\right],
\end{align*}
where $\mathcal{S}_{jk} = \{s~|~\mathbb{I}[\bm d_s(j)\neq 0,\bm d_s(k) \neq 0]\}$.

Let us construct a random variable $V_{j,s}$, where $V_{j,s}=1$ if $Z_j$ is observed in $\bm d_s$; otherwise $V_{j,s}=0$. With this definition, we can construct a derived RV $S_{jk}$ that is sum of $S$ i.i.d. Bernoulli random variables such that
$
S_{jk} = \sum_{s=1}^S \mathbb{I}[V_{j,s}=1\ {\rm and}\ V_{k,s}=1],
$
where we have $\mathbb{E}[S_{jk}] = Sp^2$.

In order to characterize the random variable $S_{jk}$, we can use Chernoff lower tail bound such that for $0<t < 1$,
\begin{align} \label{eq:Sml}
{\sf Pr}\left[ S_{jk} \ge (1-t)Sp^2 \right] \ge 1-e^{-Sp^2t^2/2}.
\end{align}

Combining {Lemma~19} in \cite{zhang2014spectral} and \eqref{eq:Sml},  we have
\begin{align} 
&{\sf Pr}\bigg[\|\widehat{\X}_{jk }-{\X}_{jk}\|_{\rm F}\le \frac{(1+\sqrt{{\rm log}(1/\delta)})}{\sqrt{(1-t)Sp^2}} \bigg]\nonumber\\
& ={\sf Pr}\bigg[\|\widehat{\X}_{jk }-{\X}_{jk}\|_{\rm F}\le \frac{(1+\sqrt{{\rm log}(1/\delta)})}{\sqrt{S_{jk}}},S_{jk}\geq (1-t)Sp^2 \bigg] \nonumber\\
&\ge 1-\delta-e^{-Sp^2t^2/2}\label{eq:Xestimprob},
\end{align}
where we have applied the De Morgan's law and the union bound to obtain the last inequality.

Since \eqref{eq:Xestimprob} holds for any $t \in (0,1)$, we set $t = 1/2$ for expression simplicity. Then, we have 
\begin{align} \label{eq:Rprob_t}
{\sf Pr}\bigg[\|\widehat{\X}_{jk }-{\X}_{jk}\|_{\rm F} &\le  \frac{\sqrt{2}(1+\sqrt{{\rm log}(1/\delta)})}{p\sqrt{S}} \bigg] \nonumber \\
&\ge 
1-\delta-e^{-Sp^2/8}. 
\end{align}

It follows that
if $p^2 \ge \frac{8}{S}\log(1/\delta)$, the right hand side of \eqref{eq:Rprob_t} is greater than $1-2\delta$. 
%\reminder{what is our rationale behind selecting $t=1/2$.}

\section{Proof of Theorem~\ref{thm:em}}\label{app:thm_em}

We first introduce the EM algorithm in \cite{yeredor2019maximum}. Then, we will show that the EM algorithm improves upon good initializations (e.g., those given by the {CNMF-SPA}).
\subsection{An EM Algorithm for Joint PMF Learning \cite{yeredor2019maximum}}
%Let $f_s \in \{1,\dots,F\}$ be the realization of the `latent variable' $H$ in the $s$th realization. The joint likelihood of the observed data $\{\bm d_s\}_{s=1}^S$ and the corresponding latent variable realizations $\{f_s\}_{s=1}^S$ as function of the latent factor $\bm \theta =\begin{bmatrix}{\rm vec}(\bm A_1)^\T,\dots,{\rm vec}(\bm A_N)^\T, \bm \lambda^{\T}\end{bmatrix}^\T$ can be expressed as 
%\begin{align*}
%{\cal L}(\{\bm d_s,f_s\}_{s=1}^S;\bm \theta)=\prod_{s=1}^S \bm \lambda(f_s) \prod_{n=1}^N\prod_{i=1}^{I_n} (\bm A_n(i,f_s))^{\mathbb{I}(\bm d_s(n)=z_{n}^{(i)})}.
%\end{align*}
The EM algorithm proposed in \cite{yeredor2019maximum} for handling maximizing the log-likelihood in \eqref{eq:loglike} has the following E-step and M-step:

\textbf{E-step}: The posterior of the latent variable ${q}_{s,f} ={\sf Pr}(f_s=f~|~\bm d_1,\dots,\bm d_S,\widehat{\bm \theta})$ is estimated via
$
\widehat{q}_{s,f} 
=\frac{\exp(\log(\widehat{\bm \lambda}(f))+\sum_{n=1}^N\sum_{i=1}^{I_n}\mathbb{I}(\bm d_s(n)=z_{n}^{(i)})\log(\widehat{\bm A}_n(i,f)))}{\sum_{f'=1}^F\exp(\log(\widehat{\bm \lambda}(f'))+\sum_{n=1}^N\sum_{i=1}^{I_n}\mathbb{I}(\bm d_s(n)=z_{n}^{(i)})\log(\widehat{\bm A}_n(i,f')))}.
$

\textbf{M-step}: Using the estimated $\widehat{q}_{s,f}$,  $\widehat{\bm A}_n$ and $\widehat{\bm \lambda}$ that maximize the likelihood function are computed using the following:
\begin{subequations}\label{eq:emM}
	\begin{align} 
	\widehat{\bm A}_n(i,f) &\leftarrow \frac{\sum_{s=1}^S\widehat{q}_{s,f}\mathbb{I}(\bm d_s(n)=z_{n}^{(i)})}{\sum_{i'=1}^{I_n}\sum_{s=1}^S\widehat{q}_{s,f}\mathbb{I}(\bm d_s(n)=z_{n}^{(i')})},\\
	\widehat{\bm \lambda}(f)&\leftarrow\frac{\sum_{s=1}^S\widehat{q}_{s,f}}{\sum_{f'=1}^{F}\sum_{s=1}^S\widehat{q}_{s,f'}}.
	\end{align}
\end{subequations}

\subsection{Performance Analysis for EM}
We employ the paradigm in \cite{zhang2014spectral} for proving the optimality of an EM algorithm. There, the EM algorithm 
learns a naive Bayes model as in \eqref{eq:pmf_latent_var} but assuming a uniform latent distribution, i.e., $\bm \lambda(f)=1/F$. In our case, we do not assume uniform prior for $\bm \lambda$ and thus our proof covers more general cases.
To begin with, we follow the idea in \cite{zhang2014spectral} to define a number of events as below:
\begin{align*}
&\mathcal{E}_1 : ~ \sum_{n=1}^N \sum_{i=1}^{I_n} \mathbb{I}(\bm d_s(n)=z_{n}^{(i)})\log\left(\frac{\bm A_n(i,f_s)}{\bm A_n(i,f)}\right) \ge N\overline{D}_1/2,\\
&\mathcal{E}_2 : ~ \left|\sum_{s=1}^S\mathbb{I}(f_s=f)\mathbb{I}(\bm d_s(n)=z_{n}^{(i)})-S\bm \lambda(f)p\bm A_n(i,f)\right| \le St_{nif}, \\
&\mathcal{E}_3 : ~\left|\sum_{s=1}^S\mathbb{I}(f_s=f)\mathbb{I}(\bm d_s(n)\neq 0)-S\bm \lambda(f)p\right| \le {St_{nif}}, \\
&\mathcal{E}_4 :~\left|\sum_{s=1}^S\mathbb{I}(f_s=f)-S\bm \lambda(f)\right| \le Sc_f, 
\end{align*}
where the events are defined for all $f,n,i,f_s$, $\bm d_s(n)\neq 0$ represents that $n$-th RV is observed with any value from its alphabet set in the $s$-th sample, and $t_{nif}>0,c_f>0$ are scalars.
Note that ${\cal E}_1,~{\cal E}_2$ and ${\cal E}_3$ are the same as those defined in \cite{zhang2014spectral}, while ${\cal E}_4$ is introduced in our proof to accommodate the general $\bm \lambda$.

First, we consider the E-step. The parameter $\widehat{q}_{s,f}$ can be bounded using the following lemma:

\begin{lemma} \label{lem:qsfbound}
	Assume that the event $\mathcal{E}_1$ happens and also assume that $\bm A_n, \bm \lambda$ and the initial estimates satisfy $|\widehat{\bm A}_n(i,f)-{\bm A}_n(i,f)| \le \delta_1$, ${\bm A}_n(i,f) \ge \rho_1$, $|\widehat{\bm \lambda}(f)-{\bm \lambda}(f)| \le \delta_2$ and ${\bm \lambda}(f) \ge \rho_2$ for all $n,i,f$. {Also assume that $\delta_1 < \rho_1$ and $\delta_2 < \rho_2$}. Then, $\widehat{q}_{s,f}$ satisfies the following:
	\begin{align} 
	&|\widehat{q}_{s,f} - \mathbb{I}(f_s=f)| \le \upsilon,~\forall f,s. \label{eq:qsf}
	\end{align}
	where $\upsilon=\exp(-({N\overline{D}} + 1-\frac{1}{\rho_2(\rho_2-\delta_2)}+ N(1-\frac{1}{\rho_1(\rho_1-\delta_1)}))+\log(F))$.
\end{lemma}
The proof of Lemma \ref{lem:qsfbound} is given in Appendix~\ref{app:lemma_qsf}.
The next lemma shows that the subsequent M-step estimates the $\bm A_n$'s and $\bm \lambda$ up to bounded errors:
\begin{lemma} \label{lem:Al}
	Assume that $\mathcal{E}_2 \bigcap \mathcal{E}_3 \bigcap \mathcal{E}_4 $ holds. Suppose $\widehat{q}_{s,f} $ satisfies the following:
	\begin{align}
	|\widehat{q}_{s,f} - \mathbb{I}(f_s=f)| \le \beta , ~\forall f,s, \label{eq:phi}
	\end{align}
	where $\beta > 0$ is a scalar. Then $\widehat{\bm A}_n$ and $\widehat{\bm \lambda}$ updated by \eqref{eq:emM} are bounded by:
	\begin{subequations}\label{eq:Aboundem2}
		\begin{align} 
		|\widehat{\bm A}_n(i,f)-\bm A_n(i,f)| &\le \frac{2t_{nif}+2S\beta}{\bm \lambda(f)p-t_{nif}-\beta}, \label{eq:Aboundem_11}\\
		|\widehat{\bm \lambda}(f)-\bm \lambda(f)|& \le \frac{c_f + \beta+ F\beta}{1-F\beta}. \label{eq:Aboundem_22}
		\end{align}
	\end{subequations}
\end{lemma}

The proof of Lemma \ref{lem:Al} is given in Appendix~\ref{app:lemma_Al}.
To proceed, we have the following lemma:

\begin{lemma} \label{lem:Alambdaem}
	Assume that $\mathcal{E}_1 \bigcap \mathcal{E}_2 \bigcap \mathcal{E}_3 \bigcap \mathcal{E}_4 $ happens. Also assume that $|\widehat{\bm A}_n^0(i,f)-{\bm A}_n(i,f)| \le \delta_1:=\frac{4}{\rho_1(4+ \overline{D})}$, ${\bm A}_n(i,f) \ge \rho_1$, $|\widehat{\bm \lambda}^0(f)-{\bm \lambda}(f)| \le \delta_2:=\frac{4}{\rho_2(4+ N\overline{D})}$ , ${\bm \lambda}(f) \ge \rho_2$ for all $n,i,f$ and $\overline{D} \ge \max\left\{\frac{8-4\rho_1^2}{\rho_1^2},\frac{8-4\rho_2^2}{N\rho_2^2}\right\}$. Suppose that the following holds $\forall g\in \{ \{t_{nif}\}_{n,i,f}, \{c_f\}_f \}$:
	\begin{align} 
	2\exp(-\frac{N\overline{D}}{2}+\log(F)) \le g &\le \frac{p\rho_2}{8F}\xi,  \label{eq:tcondition}
	\end{align}
	where $\xi=\min\left(\frac{4}{\rho_1(4+ \overline{D})},\frac{4}{\rho_2(4+ N\overline{D})}\right)$.
	%where $\overline{D} =\frac{\overline{D}_1+\overline{D}_2}{2}$. 
	Then, by updating the parameters using EM at least once (i.e., after runing the EM algorithm for at least one iteration), we have the following:
	\begin{subequations}\label{eq:Aboundem1}
		\begin{align} 
		\left|\widehat{\bm A}_n(i,f)-\bm A_n(i,f)\right| &\le \frac{4t_{nif}}{\bm \lambda(f)p} \le \delta_1, \label{eq:Aboundem_11}\\
		\left|\widehat{\bm \lambda}(f)-\bm \lambda(f)\right|& \le 8Fc_f \le \delta_2. \label{eq:Aboundem_22}
		\end{align}
	\end{subequations}

\end{lemma} 
The proof of Lemma \ref{lem:Alambdaem} is given in Appendix~\ref{app:lem_Alall}.

Let us characterize the probability of the intersection $\mathcal{E}_1 \bigcap \mathcal{E}_2 \bigcap \mathcal{E}_3 \bigcap \mathcal{E}_4 $ happening.
Theorem 4 in \cite{zhang2014spectral} characterizes the probabilities of $\mathcal{E}_1$, $\mathcal{E}_2$ and $\mathcal{E}_3$ happening. Specifically, we get the following results from \cite{zhang2014spectral}:
\begin{subequations} \label{eq:eventbounds}
	\begin{align}
	{\sf Pr}(\mathcal{E}_1) &\ge 1-  SF\exp\left({ -\frac{N\overline{D}_1}{33 \log(1/\rho_1)}}\right),\\
	{\sf Pr}(\mathcal{E}_2) &\ge 1-  \sum_{n=1}^N\sum_{f=1}^F \sum_{i=1}^{I_n} 2\exp\left(-\frac{St^2_{nif}}{3p\bm \lambda(f)}\right),\\
	{\sf Pr}(\mathcal{E}_3) &\ge 1-  \sum_{n=1}^N\sum_{f=1}^F \sum_{i=1}^{I_n} 2\exp\left(-\frac{St^2_{nif}}{3p\bm \lambda(f)}\right).
	\end{align}
\end{subequations}
In order to characterize ${\sf Pr}(\mathcal{E}_4)$, we observe that 
$\sum_{s=1}^S\mathbb{I}(f_s=f)$ is the sum of i.i.d. Bernoulli random variables. Using the Chernoff bound, for a particular $f$, we have
\begin{align}
{\sf Pr}&\left[\left|\sum_{s=1}^S\mathbb{I}(f_s=f)-S\bm \lambda(f)\right| \ge Sc_f\right] \le 2e^{-\frac{Sc_f^2}{3\bm \lambda(f)}}.
\end{align}
By taking the union bound over all $f \in \{1,\dots,F\}$, we obtain 
\begin{align} \label{eq:pre4}
{\sf Pr}(\mathcal{E}_4) \ge 1- \sum_{f=1}^F2\exp(-Sc_f^2/(3\bm \lambda(f))).
\end{align}
Applying the union bound and the De Morgan's law, one can see that $\mathcal{E}_1 \bigcap \mathcal{E}_2 \bigcap \mathcal{E}_3 \bigcap \mathcal{E}_4 $ happens with a probability greater than equal to 
$ 1-SF { e^{-\frac{N\overline{D}_1}{33 \log(1/\rho_1)}}}- \sum_{n=1}^N\sum_{f=1}^F \sum_{i=1}^{I_n} 4e^{-\frac{St^2_{nif}}{3p\bm \lambda(f)}} - \sum_{f=1}^F 2e^{-\frac{Sc_f^2}{3\bm \lambda(f)}}.$

To ensure that the estimation error bounds for $\bm A_n$ and $\bm \lambda$ given by \eqref{eq:Aboundem1} hold with probability greater than $1-\epsilon$, a sufficient condition is that the following being satisfied simultaneously:
\begin{align}
N &\ge \frac{33\log(1/\rho_1)\log(3SF/\epsilon)}{\overline{D}_1} \label{eq:Nboundem}\\
S &\ge \frac{3p\bm \lambda(f)\log(12NFI/\epsilon)}{t^2_{nif}} \label{eq:Sboundem1}\\
S &\ge \frac{3 \bm \lambda(f) \log(6F/\epsilon)}{c_f^2}. \label{eq:Sboundem2}
\end{align}
We can assign specific values to $t_{nif}$ and $c_{f}$ such that the above conditions are satisfied. Let
\begin{subequations}\label{eq:tf}
	\begin{align}
	t_{nif} &:= \sqrt{\frac{3p\bm \lambda(f)\log(12NFI/\epsilon)}{S}},\\
	c_f &:= \sqrt{\frac{3\bm \lambda(f)\log(12NFI/\epsilon)}{S}}.
	\end{align}  
\end{subequations}
By this selection of $t_{nif}$ and $c_f$, the conditions in \eqref{eq:Sboundem1} and \eqref{eq:Sboundem2} hold. 
To enforce the condition  \eqref{eq:tcondition}, the following equalities have to hold:
\begin{align*}
\sqrt{\frac{3p\bm \lambda(f)\log(12NFI/\epsilon)}{S}} &\ge 2\exp\left(-\frac{N\overline{D}}{2}+\log(F)\right)\\
\sqrt{\frac{3\bm \lambda(f)\log(12NFI/\epsilon)}{S}} &\le \frac{p\rho_2}{8F}\xi=\frac{p\rho_2\delta_{\min}}{8F},
\end{align*} 
where $\delta_{\min} =\min(\delta_1,\delta_2)$.
The above can be implied by the following:
\begin{align}
%\frac{3p\bm \lambda(f)\log(12NFI/\delta)}{S} &\ge 4F^2\exp\left(-{N\overline{D}}\right)\nonumber \\
%p &\ge \frac{4SF^2\exp\left(-{N\overline{D}}\right)}{3 \rho_2\log(12NFI/\epsilon)}, \label{eq:pbound1}\\
N &\ge \frac{4\log(2SF^2/(3p\rho_2\log(12NFI/\epsilon)))}{\overline{D}} \label{eq:Nbound2}\\
S &\ge \frac{192F^2\log(12NFI/\epsilon)}{p^2\rho_2^2\delta^2_{\min}},\label{eq:Sbound}
\end{align}
where we have used $1 \ge \bm \lambda(f) \ge \rho_2$.

Using the inequality $\log x > 1-\frac{1}{x}, x > 0$,  we can express the condition \eqref{eq:Nbound2} as
\begin{align*}
N &\ge \frac{4\log(2SF^2/(3p\rho_2(1-\epsilon/(12NFI))))}{\overline{D}}   
\end{align*}
{Since $\epsilon \le 1$ and the product $NFI > 1$, we always have $1-\epsilon/(12NFI) > \epsilon/2$,}
%\reminder{$1-\epsilon/(12NFI) > 1-\epsilon/(12FI)$ since $N > 1$.
%Then $N$ bound can be written as
%\begin{align*}
%N &\ge \frac{4\log(48SF^3I/(3p\rho_2(12FI-\epsilon)))}{\overline{D}} .  
%\end{align*}
%}
which makes the above condition as follows:
\begin{align}
N &\ge \frac{4\log(4SF^2/(3p\rho_2\epsilon))}{\overline{D}} \label{eq:Nbound3}.  
\end{align}
Combing the two conditions \eqref{eq:Nboundem} and \eqref{eq:Nbound3}, we have
\begin{align}
%N \ge \max\left(\frac{33\log(1/\rho_1)\log(3SF/\epsilon)}{\overline{D}_1},\frac{4\log(4SF^2/(3p\rho_2\epsilon))}{\overline{D}}\right)\label{eq:Nbound4}
N \ge \max\left(\frac{33\log(3SF/\epsilon)}{\rho_1\overline{D}_1},\frac{4\log(4SF^2/(3p\rho_2\epsilon))}{\overline{D}}\right)\label{eq:Nbound4},
\end{align}
where we have used the fact that $\log(1/\rho_1)\le (1/\rho_1)-1 < 1/\rho_1$.

To summarize, if \eqref{eq:Sbound} and \eqref{eq:Nbound4} hold and $t_{nif}$ and $c_f$ are chosen to be as in \eqref{eq:tf}, then, with probability at least $1-\epsilon$, the following inequalities hold by Lemma~\ref{lem:Alambdaem}:
\begin{align*}
&|\widehat{\bm A}_n(i,f)-\bm A_n(i,f)|^2 \le \frac{16t_{nif}^2}{p^2\bm \lambda(f)^2}\le \frac{48\log(12NFI/\epsilon)}{Sp\bm \lambda(f)}\\
&|\widehat{\bm \lambda}(f)-\bm \lambda(f)|^2 \le 64F^2c_f^2 \le \frac{192F^2\bm \lambda(f)\log(12NFI/\epsilon)}{S}.
\end{align*}
This completes the proof.

\section{Proofs of Lemmas in Theorem \ref{thm:em}}
\subsection{Proof of Lemma \ref{lem:qsfbound}}\label{app:lemma_qsf}
Consider the E-step. For any $f \neq f_s$, we have
\begin{align}
&\widehat{q}_{s,f} \nonumber \\
&\le \frac{\exp(\log(\widehat{\bm \lambda}(f))+\sum_{n=1}^N\sum_{i=1}^{I_n}\mathbb{I}(\bm d_s(n)=z_{n}^{(i)})\log(\widehat{\bm A}_n(i,f)))}{\exp(\log(\widehat{\bm \lambda}(f_s))+\sum_{n=1}^N\sum_{i=1}^{I_n}\mathbb{I}(\bm d_s(n)=z_{n}^{(i)})\log(\widehat{\bm A}_n(i,f_s)))}\nonumber\\
&= \frac{\widehat{\bm \lambda}(f)\prod_{n=1}^N\prod_{i=1}^{I_n} (\widehat{\bm A}_n(i,f))^{\mathbb{I}(\bm d_s(n)=z_{n}^{(i)})}}{\widehat{\bm \lambda}(f_s)\prod_{n=1}^N\prod_{i=1}^{I_n} (\widehat{\bm A}_n(i,f_s))^{\mathbb{I}(\bm d_s(n)=z_{n}^{(i)})}}\nonumber\\
&=\frac{\widehat{\bm \lambda}(f)}{\widehat{\bm \lambda}(f_s)}\prod_{n=1}^N\prod_{i=1}^{I_n} \left(\frac{ \widehat{\bm A}_n(i,f)}{ \widehat{\bm A}_n(i,f_s)}\right)^{\mathbb{I}(\bm d_s(n)=z_{n}^{(i)})}\nonumber\\
&= 1/\exp B_f, \label{eq:boundqsf}
\end{align}
where $B_f = \log (\widehat{\bm \lambda}(f_s)/\widehat{\bm \lambda}(f))+\sum_{n=1}^N\sum_{i=1}^{I_n}\mathbb{I}(\bm d_s(n)=z_{n}^{(i)})\log(\nicefrac{\widehat{\bm A}_n(i,f_s)}{\widehat{\bm A}_n(i,f)})$.
Then it follows that
\begin{align} \label{eqBf}
&B_{f} 
= \log\frac{{\bm \lambda}(f_s)}{{\bm \lambda}(f)}\frac{\widehat{\bm \lambda}(f_s)}{{\bm \lambda}(f_s)}\frac{{\bm \lambda}(f)}{\widehat{\bm \lambda}(f)} \nonumber\\
&+\sum_{n=1}^N\sum_{i=1}^{I_n}\mathbb{I}(\bm d_s(n)=z_{n}^{(i)})\log\frac{\bm A_n(i,f_s)}{\bm A_n(i,f)}\frac{\widehat{\bm A}_n(i,f_s)}{{\bm A}_n(i,f_s)}\frac{{\bm A}_n(i,f)}{\widehat{\bm A}_n(i,f)}\nonumber\\
&= \log\frac{{\bm \lambda}(f_s)}{{\bm \lambda}(f)}+\log\frac{\widehat{\bm \lambda}(f_s)}{{\bm \lambda}(f_s)}-\log\frac{\widehat{\bm \lambda}(f)}{{\bm \lambda}(f)} \nonumber\\
&+\sum_{n=1}^N\sum_{i=1}^{I_n}\mathbb{I}(\bm d_s(n)=z_{n}^{(i)})\log\frac{\bm A_n(i,f_s)}{\bm A_n(i,f)}\\
& +\sum_{n=1}^N\sum_{i=1}^{I_n}\mathbb{I}(\bm d_s(n)=z_{n}^{(i)})\left[\log\frac{\widehat{\bm A}_n(i,f_s)}{{\bm A}_n(i,f_s)}-\log\frac{\widehat{\bm A}_n(i,f)}{{\bm A}_n(i,f)}\right]. \nonumber
\end{align}

To proceed, we can bound all the terms in \eqref{eqBf}. First we have
\begin{align}
\log\frac{\widehat{\bm \lambda}(f_s)}{{\bm \lambda}(f_s)}-\log\frac{\widehat{\bm \lambda}(f)}{{\bm \lambda}(f)} 
&\ge \log\left(\rho_2-\delta_2\right)+\log\left(\rho_2\right)\nonumber\\
&=  \log\left(\rho_2(\rho_2-\delta_2)\right)\nonumber\\
&\ge  1-\frac{1}{\rho_2(\rho_2-\delta_2)}\label{eq:logl},
\end{align}
where the first inequality uses $\widehat{\bm \lambda}(f) \ge \rho_2-\delta_2$, $\bm \lambda(f) \ge \rho_2$, $\widehat{\bm \lambda}(f),\bm \lambda(f) \le 1$, { $\delta_1<\rho_1$ and $\delta_2<\rho_2$}. The last inequality is due to the fact that $\log(x) > 1-\frac{1}{x}$ for $x > 0$. 
%\reminder{$\rho_2>\delta_2$ and $\rho_1>\delta_1$?}

Similarly, we can bound
\begin{align}
\log\frac{\widehat{\bm A}_n(i,f_s)}{{\bm A}_n(i,f_s)}-\log\frac{\widehat{\bm A}_n(i,f)}{{\bm A}_n(i,f)} \ge 1-\frac{1}{\rho_1(\rho_1-\delta_1)} \label{eq:logA}.
\end{align}

Assuming that $\mathcal{E}_1$ happens, we have
\begin{align*}
B_{f} &\ge \frac{N\overline{D}_2}{2} + 1-\frac{1}{\rho_2(\rho_2-\delta_2)}+\frac{N\overline{D}_1}{2}\\
&\quad\quad+ N\left(1-\frac{1}{\rho_1(\rho_1-\delta_1)}\right)\\
&= \frac{N(\overline{D}_1+\overline{D}_2)}{2} + 1-\frac{1}{\rho_2(\rho_2-\delta_2)}\\
&\quad\quad +N\left(1-\frac{1}{\rho_1(\rho_1-\delta_1)}\right),
\end{align*}
where the first inequality is obtained by using the definitions of $\overline{D}_2$ and event $\mathcal{E}_1$, equations \eqref{eq:logl} and\eqref{eq:logA}.
Using the definition $\overline{D} := \frac{\overline{D}_1+\overline{D}_2}{2}$, one can see that
\begin{align}
B_{f} &\ge  {N\overline{D}} + 1-\frac{1}{\rho_2(\rho_2-\delta_2)}+ N\left(1-\frac{1}{\rho_1(\rho_1-\delta_1)}\right). \label{eq:boundBf}
\end{align}

Combining \eqref{eq:boundBf} with \eqref{eq:boundqsf}, for every $f \neq f_s$, we have
\begin{align}\label{eq:qsf_1}
\widehat{q}_{s,f} \le 1/\exp{B_f} \le \vartheta,
\end{align}
$\vartheta = e^{-({N\overline{D}} + 1-\frac{1}{\rho_2(\rho_2-\delta_2)}+ N(1-\frac{1}{\rho_1(\rho_1-\delta_1)}))}$.
Using \eqref{eq:qsf_1} and $\sum_{f=1}^F \widehat{q}_{s,f}=1$, we have
\begin{align}\label{eq:qsf_2}
\widehat{q}_{s,f_s} = 1-\sum_{f \neq f_s}\widehat{q}_{s,f} \ge 1- F\vartheta.
\end{align}
The inequalities in \eqref{eq:qsf_1} and \eqref{eq:qsf_2} can be summarized as follows: 
\begin{align} 
&|\widehat{q}_{s,f} - \mathbb{I}(f_s=f)| \le  \upsilon,~\forall f,s.\label{eq:qsf1}
\end{align}

\subsection{Proof of Lemma \ref{lem:Al}}\label{app:lemma_Al}
The first result given in \eqref{eq:Aboundem_11} follows similar steps as in Lemma 9 from \cite{zhang2014spectral} which is detailed below using the notations in our case:

According to the M-step update \eqref{eq:emM}, we can write
$
\widehat{\bm A}_n(i,f) = \frac{A}{B},
$
where $A := \sum_{s=1}^S\widehat{q}_{s,f}\mathbb{I}(\bm d_s(n)=z_{n}^{(i)})$ and $B := \sum_{i'=1}^{I_n}\sum_{s=1}^S\widehat{q}_{s,f}\mathbb{I}(\bm d_s(n)=z_{n}^{(i')})$.

Assuming that the event $\mathcal{E}_2$ holds true, we can have
\begin{align}
&|A-S\bm \lambda(f)p\bm A_n(i,f) | \nonumber\\
&\le |\sum_{s=1}^S\mathbb{I}(f_s=f)\mathbb{I}(\bm d_s(n)=z_{n}^{(i)})-S\bm \lambda(f)p\bm A_n(i,f)|\nonumber\\
&+|\sum_{s=1}^S\widehat{q}_{s,f}\mathbb{I}(\bm d_s(n)=z_{n}^{(i)})-\sum_{s=1}^S\mathbb{I}(f_s=f)\mathbb{I}(\bm d_s(n)=z_{n}^{(i)})|\nonumber\\
&\le St_{nif}+S\beta, \label{eq:ASbound}
\end{align}
where the last inequality is by the definition of $\mathcal{E}_2$ and \eqref{eq:phi}.
Assuming that the event $\mathcal{E}_3$ holds true, we have
\begin{align}
&|B- S\bm \lambda(f)p| \le |\sum_{s=1}^S\mathbb{I}(f_s=f)\mathbb{I}(\bm d_s(n)\neq 0)-S\bm \lambda(f)p| \nonumber\\
&+ |\sum_{i'=1}^{I_n}\sum_{s=1}^S\widehat{q}_{s,f}\mathbb{I}(\bm d_s(n)=z_{n}^{(i')})-\sum_{s=1}^S\mathbb{I}(f_s=f)\mathbb{I}(\bm d_s(n)\neq 0)|\nonumber \\
&\le St_{nif} + S \beta, \label{eq:BSbound}
\end{align}
where the last inequality is obtained by assuming that event $\mathcal{E}_3$ holds true and using \eqref{eq:phi}.

Combining the bounds for A and B, we can get
\begin{align*}
&\left|\widehat{\bm A}_n(i,f)-\bm A_n(i,f)\right| = \left|\frac{A}{B}-\bm A_n(i,f)\right|\\
&= \left|\frac{S\bm \lambda(f)p\bm A_n(i,f)+A-S\bm \lambda(f)p\bm A_n(i,f)}{S\bm \lambda(f)p+B- S\bm \lambda(f)p}-\bm A_n(i,f)\right|\\
&= \left|\frac{A-S\bm \lambda(f)p\bm A_n(i,f)-\bm A_n(i,f)(B- S\bm \lambda(f)p)}{S\bm \lambda(f)p+B- S\bm \lambda(f)p}\right|\\
&\le \frac{\left|A-S\bm \lambda(f)p\bm A_n(i,f)\right|+\bm A_n(i,f)\left|(B- S\bm \lambda(f)p)\right|}{\left|S\bm \lambda(f)p+B- S\bm \lambda(f)p\right|}\\
&\le \frac{2St_{nif}+2S\beta}{S\bm \lambda(f)p-St_{nif}-S\beta},
\end{align*}
where the last inequality is by the fact that $\bm A_n(i,f) \le 1$ and the bounds from \eqref{eq:ASbound} and \eqref{eq:BSbound}.

For the second result, consider the M-step update for $\widehat{\bm \lambda}$ given by \eqref{eq:emM}.
One can write $\widehat{\bm \lambda}(f) = C/D$ where
\begin{align*}
C = \sum_{s=1}^S\widehat{q}_{s,f}, \quad	D = \sum_{f'=1}^{F}\sum_{s=1}^S\widehat{q}_{s,f'}.
\end{align*}
Assume that the event $\mathcal{E}_4$ happens, using \eqref{eq:phi}, we have
\begin{align*}
|C-S\bm \lambda(f)| &\le \left|\sum_{s=1}^S\mathbb{I}(f_s=f)-S\bm \lambda(f)\right| \nonumber \\
&+\left|\sum_{s=1}^S\widehat{q}_{s,f}-\sum_{s=1}^S\mathbb{I}(f_s=f)\right|\\
&\le Sc_f + S\beta.
\end{align*}
In addition, we have
\begin{align*}
|D-S| &\le \left|\sum_{f'=1}^{F}\sum_{s=1}^S\mathbb{I}(f_s=f')-S\right| \nonumber\\
&+\left|\sum_{f'=1}^{F}\sum_{s=1}^S\widehat{q}_{s,f'}-\sum_{f'=1}^{F}\sum_{s=1}^S\mathbb{I}(f_s=f')\right|
\le SF\beta.
\end{align*}
Combining the bounds for $C$ and $D$, we obtain
\begin{align*}
|\widehat{\bm \lambda}(f)-\bm \lambda(f)| &= \left|\frac{C}{D} - \bm \lambda(f)\right|\\
&= \left|\frac{(C-S\bm \lambda(f))+S\bm \lambda(f)}{(D-S)+S} - \bm \lambda(f)\right|\\
&= \left|\frac{(C-S\bm \lambda(f_s))-\bm \lambda(f)(D-S)}{(D-S)+S} \right|\\
&\le \frac{Sc_f + S\beta+ SF\beta}{S-SF\beta}, 
\end{align*}
where the last inequality is by using triangle inequality and the fact that $\bm \lambda(f) \le 1$.

\subsection{Proof of Lemma \ref{lem:Alambdaem}} \label{app:lem_Alall}

%Let us fix $\phi := t_{\min}/2$ where $t_{\min} = \min\{t_{nif}\}$. Next, 
Consider the following term in \eqref{eq:qsf} from Lemma~\ref{lem:qsfbound}:
\begin{align*}
&{N\overline{D}} + 1-\frac{1}{\rho_2(\rho_2-\delta_2)}+ N\left(1-\frac{1}{\rho_1(\rho_1-\delta_1)}\right) = \\
&\quad N\left(\underbrace{\frac{ \overline{D}}{2} + 1-\frac{1}{\rho_1(\rho_1-\delta_1)}}_{E}\right)+\underbrace{\frac{ N\overline{D}}{2} + 1-\frac{1}{\rho_2(\rho_2-\delta_2)}}_{G}.
\end{align*}
In order to bound the term $E$ as below
\begin{align*}
E := \frac{ \overline{D}}{2} + 1-\frac{1}{\rho_1(\rho_1-\delta_1)} \ge \frac{ \overline{D}}{4},
\end{align*} 
$\delta_1$ has to be bounded such that
\begin{align}
\delta_1 &\le \rho_1 - \frac{4}{\rho_1(4+ \overline{D})}= \frac{4\rho_1^2+\overline{D}\rho_1^2-4}{\rho_1(4+ \overline{D})}. \label{eq:delta1}
\end{align}
Similarly, in order to bound $G$ as below
\begin{align*}
G := \frac{ N\overline{D}}{2} + 1-\frac{1}{\rho_2(\rho_2-\delta_2)} \ge \frac{ N\overline{D}}{4},
\end{align*}
$\delta_2$ has to be bounded such that
\begin{align}
\delta_2 &\le \rho_2 - \frac{4}{\rho_2(4+ N\overline{D})} = \frac{4\rho_2^2+N\overline{D}\rho_2^2-4}{\rho_2(4+ N\overline{D})}.\label{eq:delta2}
\end{align}
One can pick the values to $\delta_1$ and $\delta_2$ such that the conditions \eqref{eq:delta1} and \eqref{eq:delta2} are satisfied. 
Since $\overline{D} \ge \max\left\{\frac{8-4\rho_1^2}{\rho_1^2},\frac{8-4\rho_2^2}{N\rho_2^2}\right\}$, $  4\rho_1^2+\overline{D}\rho_1^2 \ge 8$ and $ 4\rho_2^2+N\overline{D}\rho_2^2 \ge 8$ hold true. 

Therefore, one can see that the defined $\delta_1$ and $\delta_2$ in the lemma's statement, i.e.,
\begin{align}
\delta_1 := \frac{4}{\rho_1(4+ \overline{D})} ,\quad \quad \delta_2 := \frac{4}{\rho_2(4+ N\overline{D})}  \label{eq:fixdelta1}
\end{align}
satisfy the conditions \eqref{eq:delta1} and \eqref{eq:delta2}. 
Indeed, since $\overline{D} \ge \frac{8-4\rho_1^2}{\rho_1^2}$, $\delta_1 := \frac{4}{\rho_1(4+ \overline{D})} \le \frac{\rho_1}{2} \le 1$, the above is valid. Similarly, we have $\delta_2 := \frac{4}{\rho_2(4+ N\overline{D})} \le \frac{\rho_2}{2} \le 1$.
Therefore, using the assigned values of $\delta_1$ and $\delta_2$, we have 
\begin{align} \label{eq:NDbar}
{N\overline{D}} + 1-\frac{1}{\rho_2(\rho_2-\delta_2)}+ N\left(1-\frac{1}{\rho_1(\rho_1-\delta_1)}\right) \ge \frac{N\overline{D}}{2}.
\end{align}
%\begin{align}
% \overline{D} \ge \frac{8-4\rho_1^2}{\rho_1^2}.
%\end{align}
%The, we can re-write the condition as
%\begin{align*}
%\delta_1 \le  \frac{4}{\rho_1(4+ \overline{D}(N))}.
%\end{align*}
%Let us fix the values such that
%\begin{align*}
%\delta_1 =  \frac{4}{\rho_1(4+ \overline{D}(N))}, \quad \delta_2 =  \frac{4}{\rho_2(4+ \overline{D}(N))}
%\end{align*}
%In order to that, we need to have the below condition
%\begin{align*}
%\overline{D}(N) \ge \max\left\{\frac{8-4\rho_1^2}{\rho_1^2},\frac{8-4\rho_2^2}{\rho_2^2}\right\}
%\end{align*}

%To proceed, we assign $\delta_1 = \frac{\rho_1\overline{D}_1}{8+\overline{D}_1}$, $\delta_2 = \frac{\rho_1\overline{D}_2}{8+\overline{D}_2}$ and $\phi = t_{\min}/2$ where $t_{\min}$ is the smallest among $\{t_{nif}\}$ and $\{c_f\}$. 
%Then the following conditions hold:
%\begin{align*}
%\frac{\overline{D}_1}{2} - \frac{2\delta_1}{\rho_1-\delta_1} &\ge \frac{\overline{D}_1}{4},\\
%\overline{D}_2-\frac{2\delta_2}{\rho_2-\delta_2} &\ge \frac{\overline{D}_2}{2},
%\end{align*}
Then, we have
\begin{align}
|\widehat{q}_{s,f} - \mathbb{I}(f_s=f)| &\le \upsilon \nonumber\\
&\le \exp\left(-\frac{N\overline{D}}{2}+\log(F)\right)\nonumber \\
& \le t_{\min}/2 \label{eq:tmin},
\end{align}
where $t_{\min} = \min\{\{t_{nif}\},\{c_f\}\}$ and the first inequality is by Lemma \ref{lem:qsfbound}, the second inequality is by using \eqref{eq:NDbar} and the third inequality is from the condition \eqref{eq:tcondition}.

From \eqref{eq:tcondition}, we have $t_{nif} \le \frac{p\rho_2}{8} \le \frac{p\bm \lambda(f)}{8} $. Combining this with Lemma~\ref{lem:Al}, we obtain
\begin{align*}
\left|\widehat{\bm A}_n(i,f)-\bm A_n(i,f)\right| \le \frac{2St_{nif}+2S\beta}{(7/8)S\bm \lambda(f)p-S\beta},
\end{align*}
where $\beta$ is defined such that $	|\widehat{q}_{s,f} - \mathbb{I}(f_s=f)| \le \beta , ~\forall f,s$.

From \eqref{eq:tmin}, the scalar $\beta$ can be assigned with a value such that $\beta =  t_{\min}/2$. Then, 
\begin{align*}
\left|\widehat{\bm A}_n(i,f)-\bm A_n(i,f)\right| &\le \frac{2St_{nif}+St_{\min}}{(7/8)S\bm \lambda(f)p-St_{\min}/2}\\
&\le \frac{3St_{nif}}{(7/8)S\bm \lambda(f)p-St_{\min}/2},
\end{align*}
where the last step is obtained by using $t_{\min} \le t_{nif}$. By using the condition $t_{\min} \le \frac{p\rho_2}{8} \le \frac{p\bm \lambda(f)}{8}$ from \eqref{eq:tcondition},  we get
\begin{align*}
\left|\widehat{\bm A}_n(i,f)-\bm A_n(i,f)\right| &\le \frac{3St_{nif}}{(7/8)S\bm \lambda(f)p-S\bm \lambda(f)p/16}\le \frac{4t_{nif}}{\bm \lambda(f)p}.
\end{align*}
Since
\begin{align*}
\frac{4t_{nif}}{\bm \lambda(f)p} \le \frac{4t_{nif}}{p\rho_2} \le \frac{1}{2F}\xi\le \frac{4}{\rho_1(4+ \overline{D})} = \delta_1,
\end{align*}
the estimation error of the newly updated $\bm A_n$ as given in \eqref{eq:emM} is at least no worse than the initial estimation error $\delta_1$.

%To proceed, we assign $\delta_1 = \frac{4}{\rho_1(4+ \overline{D}(N))}$, $\delta_2 = \frac{4}{\rho_2(4+ \overline{D}(N))}$ and $\phi = t_{\min}/2$ where $t_{\min}$ is the smallest among $\{t_{nif}\}$ and $\{c_f\}$. 
%Then the following conditions hold:

Next, consider the result \eqref{eq:Aboundem_22} from Lemma \ref{lem:Al}. Assigning $\beta = t_{\min}/2$,
\begin{align*}
&|\widehat{\bm \lambda}(f)-\bm \lambda(f)| \nonumber\\
&\le \frac{Sc_f + S\beta+ SF\beta}{S-SF\beta}\le \frac{Sc_f + St_{\min}/2+ SFt_{\min}/2}{S-SFt_{\min}/2}\\
&\le \frac{2Sc_f + St_{\min}+ SFt_{\min}}{2S-SFt_{\min}}\le \frac{2Sc_f + Sc_f+ SFc_f}{2S-SFt_{\min}}\\
&\le \frac{4SFc_f}{2S-SFt_{\min}}\le \frac{4SFc_f}{S-SFt_{\min}}\le 8Fc_f,
\end{align*}
where we have used the fact that $t_{\min} \le 1/2F$ according to \eqref{eq:tcondition}.

The above inequality also implies that 
\begin{align*}
|\widehat{\bm \lambda}(f)-\bm \lambda(f)| \le 8Fc_f \le {p\rho_2}\xi \le \frac{4}{\rho_2(4+ N\overline{D})} = \delta_2.
\end{align*}
That is, the estimation error of the newly updated $\bm \lambda$ as given in \eqref{eq:emM} is at least no worse than the initial estimation error $\delta_2$.

\section{Lemmata}\label{app:lemmata}
In this section, we present a collection of lemmata that will be used in our proofs.

\begin{comment}

\begin{lemma} \cite{pimentel2015charac} \label{thm:completion}
	Given a matrix, let $n(\cdot)$ denote its number of columns and $m(\cdot)$ the number
	of its nonzero rows.
	Let $\bm P$ denotes the missing entries in matrix $\bm X \in \mathbb{C}^{N_1\times N_2}$ which has rank $F$, such that $\bm P(l,k)=1$ if the $(l,k)$th entry in $\X$ is observed, and $\bm P(l,k)=0$ otherwise., Suppose every column of $\bm X$ has at least $F+1$ entries observed. Then all the missing entries of $\bm X$ can uniquely recovered if $\bm P$ contains two disjoint submatrices $\widetilde{\bm P}$ of size $N_1 \times F(N_1-F)$ and $\widehat{\bm P}$ with $N_1 \times (N_1-F)$ such that $\widetilde{\bm P}$ satisfies
	
	(i) Every matrix $\bm P^*$ formed with a subset of columns of $\widetilde{\bm P}$ satisfies 
	\begin{align}
	m({\bm P}^*) \ge n({\bm P}^*)/F +F;
	\end{align}
	
	(ii) Every matrix $\bm P^{'}$ formed with a subset of columns of $\widehat{\bm P}$ satisfies 
	\begin{align}
	m(\bm P^{'}) \ge n(\bm P^{'}) +F.
	\end{align}
\end{lemma}

\begin{lemma}\cite{pimentel2015charac} \label{lemm:completion}
	Let $\widetilde{\bm P}$ be a $N_1 \times F(N_1-F)$ matrix formed with a subset of the columns in
	$\bm P$. Suppose $\widetilde{\bm P}$ can be partitioned into $F$ matrices $\{\widehat{\bm P}_{\tau}\}_{\tau=1}^{F}$, each of size $N_1 \times (N_1-F)$,
	such that (ii) in Theorem \ref{thm:completion} holds for every $\widehat{\bm P}_{\tau}$  . Then $\widetilde{\bm P}$ satisfies (i).
\end{lemma}

\end{comment}

\begin{lemma}\label{lem:gillis} \cite{Gillis2012}
	Under the described NMF model in Eq.~\eqref{eq:noisynmf}, assume that $\| \bm{N}(:,q) \|_2 \leq \delta$ for all $q \in \{1,\dots,K\}$, if the below holds:
	\begin{equation*}
	\begin{aligned}
	\epsilon \le \sigma_{\rm min}(\bm{W})\text{\rm min}\left(\frac{1}{2\sqrt{F-1}},\frac{1}{4}\right)\left(1+80\kappa^2(\bm{W})\right)^{-1},
	\label{n2}
	\end{aligned}
	\end{equation*}
	then SPA identifies an index set $\widehat{\bm \varLambda}=\{\widehat{l}_1,\ldots\widehat{l}_F \}$ such that
	\begin{equation} \label{errorbound}
	\begin{aligned}	
	\max_{1 \leq f \leq F} \min_{\widehat{l}_f \in \widehat{\varLambda}} \left\|\bm{W}(:,f) - \widetilde{\bm{X}}(:,\widehat{l}_f)\right\|_2 \leq \epsilon\big(1+80\kappa^2(\bm{W})\big),
	\end{aligned}
	\end{equation}
	where $\kappa(\bm{W}) = \frac{\sigma_{\max}(\bm{W})}{\sigma_{\min}(\bm{W})}$ is the condition number of $\bm{W}$. 
\end{lemma}

\begin{lemma}  \cite{ibrahim2019crowdsourcing} \label{lem:normalization}
	Consider a vector $\bm x \in \mathbb{R}^{L}$ and the corresponding estimate of the vector $\widehat{\bm x} $ such that $\widehat{\bm x} = \bm x +\bm n$ where $\bm n$ represents the noise vector and $\bm x,\widehat{\bm x} \ge 0$. Assume that $\|\widehat{\bm x}\|_1 \ge \eta$ where $\eta > 0$ and $\|\bm n\|_1 < \|\bm x\|_1$. Suppose, the vector $\widehat{\bm x}$ is normalized with respect to its $\ell_1$ norm. The normalized version can be represented as
	$$\frac{\widehat{\bm x}}{\|\widehat{\bm x}\|_1} = \frac{\bm x}{\| \bm x\|_1} +\widetilde{\bm n},$$
	where $\|\widetilde{\bm n}\|_1 \le \frac{2\|\bm n\|_1}{\eta}.$
\end{lemma}

\begin{lemma}\label{lem:Lm}\cite{ibrahim2019crowdsourcing}
	Let $\rho > 0 ,\varepsilon >0$, and assume that the rows of ${\bm H}\in\mathbb{R}^{K\times F}$ are generated within the $(F-1)$-probability simplex uniformly at random (and then nonnegatively scaled). 
	\begin{itemize}
		\item[(a)] If the number of rows satisfies
		\begin{align} \label{eq:thmLm}
		K = \Omega\left(\frac{\varepsilon^{-2(F-1)}}{F}{\rm log}\left(\frac{F}{\rho}\right)\right),
		\end{align}
		then, with probability greater than or equal to $1-\rho$, there exist rows of ${\bm H}$ indexed by $ l_1,\dots  l_F$ such that 
		$$\|{\bm H}( l_f, :) - \bm e_f^\T \|_2 \le \varepsilon,~f=1,\ldots,F.$$
		\item[(a)] 	{
			Let $\frac{\alpha}{2} > \varepsilon > 0$. If $K$ satisfies
			\begin{equation}\label{eq:Mss}
			K = \Omega \left(\frac{(F-1)^2}{F\alpha^{2(F-2)}\varepsilon^{2}}{\rm log}\left(\frac{F(F-1)}{\rho}\right)\right),
			\end{equation}
			where $\alpha=1$ for $K=2$, $\alpha=2/3$ for $K=3$ and $\alpha=1/2$ for $K>3$,	
			then with probability greater than or equal to $1-\rho$, ${\H}$ is $\varepsilon$-sufficiently scattered [see detailed definition in \cite{ibrahim2019crowdsourcing}]. }
	\end{itemize}
\end{lemma}

\end{document}